\documentclass[letterpaper]{article} % DO NOT CHANGE THIS
\usepackage[]{aaai2026}  % DO NOT CHANGE THIS
\usepackage{times}  % DO NOT CHANGE THIS
\usepackage{helvet}  % DO NOT CHANGE THIS
\usepackage{courier}  % DO NOT CHANGE THIS
\usepackage[hyphens]{url}  % DO NOT CHANGE THIS
\usepackage{graphicx} % DO NOT CHANGE THIS
\usepackage{natbib}  % DO NOT CHANGE THIS AND DO NOT ADD ANY OPTIONS TO IT
\usepackage{caption} % DO NOT CHANGE THIS AND DO NOT ADD ANY OPTIONS TO IT
\usepackage{amsmath,amssymb,amsfonts}
\usepackage{booktabs}
\usepackage{algorithm}
\usepackage{algorithmic}

\newtheorem{theorem}{Theorem}
\newtheorem{proposition}{Proposition}

\title{Physics Closure Matters for Machine Olfaction: A Maxwell--Stefan Graph Solver for Identifiable Dynamic Gas Unmixing}
\author{
    Yue Shi\textsuperscript{\rm 1},
    Liangxiu Han\textsuperscript{\rm 1,*},
    Xin Zhang\textsuperscript{\rm 1},
    Tam Sobeih\textsuperscript{\rm 1}
}

\affiliations{
    \textsuperscript{\rm 1}Department of Computing and Mathematics, Faculty of Science and Engineering, Manchester Metropolitan University, Manchester M1 5GD, U.K.\\
    \textsuperscript{*}Corresponding author: l.han@mmu.ac.uk.\\
}

\begin{document}
\nocopyright
\maketitle

\begin{abstract}

Machine olfaction for gas unmixing is an underconstrained inverse problem in which gas compositions must be inferred from low-dimensional, delayed, and entangled sensor responses produced by interacting chemical transport, surface adsorption, and sensor transduction. 
One of the key obstacles is physics closure misspecification, where a neural network is designed to fit sensor traces rather than infer a physically closed olfactory process.
In this work, we formulate gas unmixing as a multi-physics-constrained inverse problem governed by Maxwell--Stefan multicomponent transport PDEs, competitive adsorption ODEs, and nonlinear sensor transduction ODEs. 
Directly solving such a high-dimensional coupled system is computationally expensive and often numerically unstable. 
To this end, we propose UnMixNet, a physics-closed graph neural solver that embeds this multi-physics forward process into end-to-end gas unmixing.
UnMixNet discretizes Maxwell--Stefan cross-diffusion on spatial graphs and formulates the multicomponent flux on each edge. 
This design enables local, differentiable, and flux-conservative inference for multicomponent cross-diffusion. 
Evaluations on SmellNet show improved single-odor recognition, seen-mixture unmixing, and unseen-mixture generalization. In addition, an external validation on UCI Dynamic Gas Mixtures shows that the inferred concentration process agrees with ground truth concentration set points under dynamic transitions. 
Process-consistency diagnostics further show that the proposed model learns transferable dynamic physical fingerprints that better satisfy transport, conservation, adsorption, and readout closure.
\end{abstract}

\section{Introduction}
Machine olfaction is becoming a frontier problem for artificial intelligence (AI), which aims to infer the composition of an odor scene from sensor responses. 
Unlike computer vision signals, olfactory sensor signals are the endpoint of an interacting physical process of odorants co-propagating, adsorbing, and producing nonlinear and entangled electrical responses. 
In realistic sensing environments, odor components do not propagate independently.
Transport is governed by multicomponent cross-diffusion, surface interaction is affected by competitive adsorption, and the final electrical readout is produced through nonlinear sensor transduction \citep{nikolic2020semiconductor,mei2024crosssensitivity}. 
When several gases co-propagate inside a sensing chamber, the response of one sensor channel may be shaped not only by the concentration of a target compound, but also by flux coupling and adsorption competition induced by other compounds \citep{lee2023principal,sung2024datacentric}. 
This makes gas unmixing a highly underconstrained inverse problem, where the goal is not merely to classify sensor patterns, but to recover physically meaningful gas compositions from dynamic and entangled observations.

\begin{figure*}[t]
    \centering
    \includegraphics[width=0.95\textwidth]{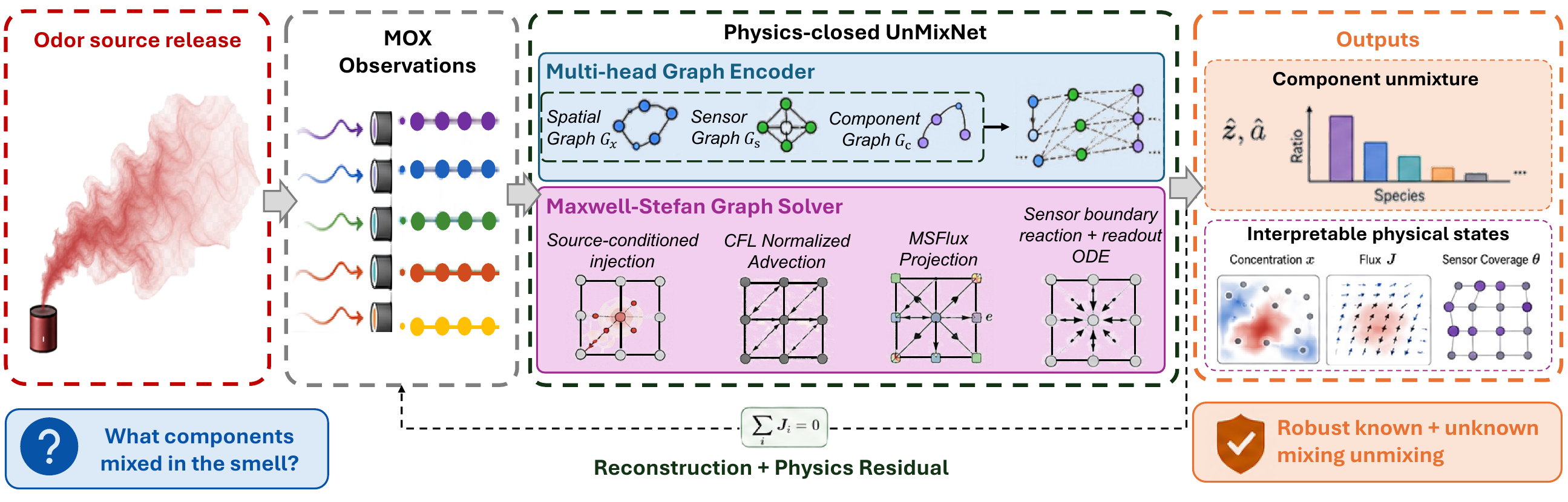}
    \caption{
    \textbf{Overview of physics-closed UnMixNet for gas unmixing.}
    Given a released odor source, a metal-oxide sensor array records dynamic responses to gas components.
    UnMixNet formulates olfactory perception as a multi-physics inverse problem. 
    A multi-head graph encoder builds spatial, sensor, and component graphs, while the Maxwell--Stefan graph solver performs source-conditioned injection, cross-diffusive graph advection, conservative multicomponent flux projection, and sensor-boundary reaction-readout dynamics.
    }
    \label{fig:overview}
\end{figure*}

A central difficulty is that direct temporal prediction is not the same as olfactory scene explanation. 
Learning-based olfactory models commonly treat sensor responses as generic time-series data and optimize a direct map from traces to labels \citep{feng2026smellnet}. 
Although such models can be effective on seen odor classes, they rarely require the predicted mixture to be explainable by mass-conservative transport, Maxwell--Stefan cross-diffusion, competitive adsorption, and nonlinear transduction dynamics. 
We refer to this failure mode as \emph{physics closure misspecification}: a model fits sensor traces while leaving the forward olfactory process that could have generated them unconstrained. 
This failure mode is especially harmful for unseen mixtures, where several component supports can induce similar delayed MOX responses unless their dynamic fingerprints are separated by a closed forward mechanism.

It is important to distinguish physics closure from standard physics-informed regularization. 
In residual-based physics-informed learning, physical equations usually enter as auxiliary losses applied to an otherwise flexible predictor. 
In contrast, UnMixNet uses physical operators as the state-transition mechanism itself. 
An inverse prediction is therefore not only a mixture-ratio vector; it must be liftable into a forward trajectory that satisfies conservative transport, Maxwell--Stefan flux admissibility, competitive adsorption, and monotone readout dynamics. 
Physics-based solvers provide a mechanistic description of gas transport and sensor dynamics, but directly solving coupled Maxwell--Stefan PDEs, competitive adsorption ODEs, and nonlinear transduction ODEs for every test sample is computationally expensive and often numerically unstable \citep{eliasof2024feature}.

In this work, we address this gap by proposing physics-closed UnMixNet, a graph-based differentiable solver for machine olfaction.
UnMixNet discretizes the physically closed transport--adsorption--readout process into a coupled spatial--component--sensor graph. Specifically, an inverse encoder infers latent variables and physical coefficients for the graph construction, while a forward graph solver reconstructs the sensor trajectory through multi-physics closure of conservative transport, Maxwell--Stefan flux projection, sensor-boundary adsorption, and nonlinear readout. 
The key methodological shift is that graph messages are not arbitrary hidden features; they are physically meaningful fluxes, boundary reactions, and readout states. 
This turns olfactory unmixing from temporal pattern regression into local, differentiable, and flux-conservative physical scene solving.

Our contributions are as follows. 
\textbf{(1) Multi-physics graph closure.} We formulate gas unmixing as a multi-physics-constrained inverse problem and introduce UnMixNet, a coupled spatial--component--sensor graph solver that jointly models source release, conservative transport, Maxwell--Stefan cross-diffusion, competitive adsorption, and monotone readout. 
\textbf{(2) Maxwell--Stefan graph flux projection.} We replace unconstrained graph messages with edge-wise multicomponent fluxes by learning symmetric positive diffusivities, constructing component Laplacians, and solving differentiable KKT projections. 
\textbf{(3) Identifiable dynamic fingerprints.} We connect physical closure to sparse mixture identifiability through dynamic fingerprint incoherence and evaluate the mechanism on real SmellNet unmixing, external concentration validation on UCI Dynamic Gas Mixtures, process-consistency diagnostics, ablations, and efficiency analysis.

\section{Related Work}

\textbf{Machine olfaction as chemical scene perception.}
Data-driven olfactory systems aim to standardize sensor response manifolds and learn robust representations \citep{sung2024datacentric}. 
Recent work learns latent odor representations to extract discriminative waveform features for gas classification \citep{sung2024datacentric} and uses deep sequence models to predict gas concentrations \citep{Yang2024AIPortableENose}. 
For example, SmellNet provides portable sensor time series for natural foods and controlled mixtures, pairs ingredients with GC--MS priors, and introduces ScentFormer for temporal smell recognition \citep{feng2026smellnet}.
These works establish machine olfaction as a data-rich and practically important sensing problem. 
However, most models still treat sensor outputs as generic temporal patterns. 
They do not explicitly require a predicted mixture to be generated by coupled gas transport, sensor-surface interaction, and transduction dynamics. 
This limits generalization to unseen mixtures, where the task is not simply to recognize a known odor trace but to explain a new chemical composition from entangled observations.

\textbf{Multi-physics mechanisms in olfactory sensing.}
Multicomponent gas transport is commonly described by Maxwell--Stefan-type diffusion models, often coupled with adsorption equilibrium. 
Recent works have developed efficient Maxwell--Stefan formulations for gas separation with multicomponent convection--diffusion systems \citep{Aznaran2025FEMOSM}.
In sensing scenarios, adsorption and diffusion kinetics have also been exploited to improve volatile-compound selectivity in nanoporous sensors \citep{Matavz2025KineticSelectivity}.
Classical Maxwell--Stefan solvers require solving coupled nonlinear PDEs with specified boundary conditions and transport parameters \citep{Huo2024MSCahnHilliard}.
Such solvers are physically expressive but computationally expensive, numerically sensitive, and difficult to insert into end-to-end inverse inference. 
UnMixNet keeps the mechanistic structure of multicomponent transport and sensor response, but amortizes the inverse problem through a differentiable graph solver.

\textbf{Neural physical solvers and graph dynamics.}
Graph-based neural solvers have recently emerged as a powerful interface between numerical simulation and representation learning. 
GRAND interprets GNN layers as discretizations of diffusion PDEs \citep{chamberlain2021grand}; message-passing neural PDE solvers replace components of numerical solvers with learned message functions \citep{brandstetter2022message}. 
Neural operators such as FNO \citep{li2021fno}, Transolver \citep{wu2024transolver}, UPT \citep{alkin2024uptfixed}, and CoDA-NO \citep{rahman2024coda} learn mappings between functions for PDE families, complex geometries, and multiphysics systems.
At a larger scale, graph neural simulators have also shown strong performance in physical forecasting tasks, such as global weather prediction \citep{Lam2023GraphCast}. 
These works show that graph structures are well suited for irregular domains and non-Euclidean discretizations.
UnMixNet differs in problem direction and message semantics: it addresses olfactory inverse sensing, and its graph messages represent species moles, Maxwell--Stefan fluxes, and sensor-boundary states rather than unconstrained latent features.

\textbf{Physics-informed learning versus physics closure.}
Physics-informed neural models often impose physical laws through residual penalties, while neural operators learn surrogate maps for families of PDE solutions. 
For dynamic gas unmixing, however, the main challenge is not only to regularize a predictor or approximate a forward simulator. 
The inferred mixture itself must be closed by a forward explanation that is admissible under the assumed transport--adsorption--readout mechanism. 
UnMixNet therefore embeds this chain into the state transition: neural networks propose physical coefficients, while graph operators enforce conservative transport, Maxwell--Stefan flux admissibility, adsorption competition, and monotone readout. 
This operator-level closure is the main distinction from generic physics-informed GNNs and from graph dynamics models whose messages remain unconstrained latent features.

\section{Problem Formulation}

Let $Y=\{y_m(t_n)\}_{n=1,m=1}^{T,M}\in\mathbb{R}^{T\times M}$ be a multichannel sensor sequence, and let $e$ denote experimental conditions such as temperature, humidity, flow rate, chamber geometry, source protocol, and sensor layout. 
There are $K$ candidate odorants and an additional carrier-air component indexed by $0$. 
For odorant $i\in\{1,\ldots,K\}$, let $z_i\in\{0,1\}$ indicate presence and let $a_i\geq 0$ denote its mixture proportion, with $\sum_{i=1}^K a_i=1$ over present components. 
We seek an amortized inverse map
\begin{equation}
(\hat z,\hat a,\hat\Psi)=E_\phi(Y,e),
\end{equation}
where $\hat\Psi$ contains physical parameters such as release rates, Maxwell--Stefan coefficients $D_{ij}$, adsorption and desorption rates $k^{\mathrm{ads}}_{m,i},k^{\mathrm{des}}_{m,i}$, transduction gains $\beta_{m,i}$, and sensor time constants $\tau_m$. 
The forward reconstruction is produced by a differentiable solver
\begin{equation}
\hat Y=F^{\mathrm{UnMixNet}}_\theta(\hat z,\hat a,\hat\Psi,e).
\end{equation}
Training closes the inverse-forward loop: inferred mixture variables must not only match labels when they are available, but also reconstruct the observed dynamic response through a physically constrained solver.

\textbf{Closure misspecification in machine olfaction.}
A direct olfactory predictor maps $(Y,e)$ to $(\hat z,\hat a)$ and can be optimized to fit mixture labels or sensor traces:
\begin{equation}
(\hat z,\hat a)=g_\varphi(Y,e).
\end{equation}
Such a predictor is closure-misspecified when its output is not required to be generated by any coupled transport--adsorption--readout process. 
This distinction is crucial for unseen mixtures. 
A model may separate familiar temporal response shapes while still assigning mass to spurious components when several odorants produce similar delayed sensor dynamics. 
In a closure-consistent inverse model, the predicted composition is admissible only if there exists a latent rollout
\begin{equation}
S_{0:T}=\{N^t,x^t,J^t,\theta^t,y^t\}_{t=0}^{T}
\end{equation}
that reconstructs the measured response through the forward solver:
\begin{equation}
S^{t+1}=\mathcal T_\theta(S^t;\hat z,\hat a,\hat\Psi,e),\qquad \hat Y\approx Y.
\end{equation}
UnMixNet implements this requirement by coupling the encoder with a graph-discretized physical solver. 
The prediction is therefore judged not only by label agreement, but also by whether the inferred mixture can explain the observed olfactory trajectory through conservative multicomponent transport, competitive sensor-boundary reaction, and nonlinear readout.

\section{Theoretical Insights}

This section explains why physics closure is useful for gas unmixing. 
We first write the olfactory forward process as a coupled multi-physics inverse problem. 
We then summarize closure guarantees that connect the graph solver to mixture identifiability, finite-volume conservation, Maxwell--Stefan admissibility, and stable sensor readout.

\subsection{Olfactory Unmixing as a Multi-Physics Inverse Problem}

\textbf{Continuous forward model.} The physical chain contains four coupled components. 
A source releases odorant $i$ according to
\begin{equation}
q_i(x,t)=z_i a_i \rho_i(t;\psi_i,T,H)p_{\mathrm{src}}(x),
\end{equation}
where $p_{\mathrm{src}}$ is a source distribution and $\rho_i$ is a positive release profile. 
Let $x_i(x,t)$ be the mole fraction, $c(x,t)$ the total molar concentration, $u(x,t)$ the advective velocity, and $J_i(x,t)$ the diffusive flux. 
Transport follows
\begin{equation}
\partial_t(c x_i)+\nabla\cdot(c x_i u+J_i)=q_i-r_i^{\mathrm{wall}}-r_i^{\mathrm{ads}},
\end{equation}
with constraints $\sum_{i=0}^{K}x_i=1$ and $\sum_{i=0}^{K}J_i=0$. 
Maxwell--Stefan diffusion is given by
\begin{equation}
-\nabla x_i=\sum_{j\neq i}\frac{x_jJ_i-x_iJ_j}{cD_{ij}}, \qquad D_{ij}=D_{ji}>0.
\end{equation}
At sensor $m$, the surface coverage $\theta_{m,i}$ follows competitive adsorption:
\begin{equation}
\dot\theta_{m,i}=k^{\mathrm{ads}}_{m,i}c_i(x_m,t)\theta_{m,0}-k^{\mathrm{des}}_{m,i}\theta_{m,i}, \quad
\theta_{m,0}=1-\sum_{i=1}^K\theta_{m,i}.
\end{equation}
The electrical signal is a nonlinear transduction of coverage with readout lag:
\begin{equation}
\tau_m(T,H)\dot{\hat y}_m=-\hat y_m+b_m(T,H)+g_m\Big(\sum_i\beta_{m,i}\theta_{m,i},T,H\Big).
\end{equation}

The equations above define the physical process that a closure-consistent inverse prediction should explain. 
The source term makes the static mixture vector dynamic; the transport equation couples odorants through Maxwell--Stefan cross-diffusion; the adsorption ODE models sensor surfaces as competitive reactive boundaries; and the readout ODE turns coverage into a delayed electrical response. 
Directly solving this continuous system for every test sample is impractical, but ignoring it creates closure misspecification. 
UnMixNet approximates this chain with local graph operators whose algebraic properties are summarized next.

\subsection{Closure Guarantees}

The proofs follow arguments from sparse recovery, finite-volume conservation, constrained graph Laplacians, and Markov generators.

\textbf{Dynamic fingerprint identifiability.} Define the pure-component dynamic fingerprint
\begin{equation}
H_i=\mathrm{vec}\big(F_\theta^{\mathrm{UnMixNet}}(z_i=1,a_i=1,e)\big)\in\mathbb{R}^{TM},
\end{equation}
and $H=[H_1,\ldots,H_K]$. 
In a local low-concentration regime, $\mathrm{vec}(Y)\approx Ha+\epsilon$. 
Let
\begin{equation}
\mu(H)=\max_{i\neq j}\frac{|H_i^\top H_j|}{\|H_i\|_2\|H_j\|_2}
\end{equation}
be the mutual coherence.

\begin{theorem}[Identifiability]
If the true mixture is $s$-sparse and $\mu(H)<1/(2s-1)$, then the $s$-sparse mixture ratio is uniquely identifiable in the noiseless linearized model.
\end{theorem}
\textbf{Proof sketch.} For any support $S$ with $|S|\leq 2s$, normalize the columns and consider $G_S=H_S^\top H_S$. 
The Gershgorin circle theorem gives $\lambda_{\min}(G_S)\geq 1-(|S|-1)\mu(H)>0$; hence every $2s$-column submatrix is full rank. 
If two $s$-sparse vectors map to the same observation, their difference is $2s$-sparse and lies in the nullspace, which yields a contradiction. 

\paragraph{Scope of the identifiability result.}
Theorem~1 should be interpreted as a local identifiability argument under a low-concentration linearized regime, where the sensor trajectory can be approximated by a sparse combination of pure-component dynamic fingerprints. 
It does not establish global identifiability for the full nonlinear transport--adsorption--readout system. 
Its role is to explain why a physics-closed decoder should produce incoherent dynamic fingerprints: lower coherence reduces support ambiguity in sparse mixture recovery. 
We therefore include an incoherence penalty during training.

\textbf{Conservative graph advection.} Let $N^n_{v,i}$ be the moles of species $i$ at spatial cell $v$. 
A graph advection layer uses transfer probabilities $P^n_{u\rightarrow v,i}\geq0$ with $\sum_vP^n_{u\rightarrow v,i}\leq1$:
\begin{equation}
N^{A}_{v,i}=\Big(1-\sum_w P^n_{v\rightarrow w,i}\Big)N^0_{v,i}+\sum_u P^n_{u\rightarrow v,i}N^0_{u,i}.
\end{equation}

\begin{proposition}
For a closed graph without sources or sinks, graph advection preserves positivity and species-wise mass: $N^A_{v,i}\geq0$ and $\sum_vN^A_{v,i}=\sum_vN^0_{v,i}$.
\end{proposition}
The proof follows from nonnegative coefficients and cancellation between outgoing and incoming edge sums. 
This property prevents the inverse model from explaining a sensor trace through hidden feature updates that would violate species-wise mass conservation.

\textbf{Maxwell--Stefan graph flux.} For edge $e=(u,v)$, define $g_{e,i}=(x_{v,i}-x_{u,i})/\ell_e$ and the edge-averaged mole fraction $\bar x_e$. 
Let $w_{e,i}$ be the relative molar velocity and define a component-graph Laplacian:
\begin{equation}
(L^c_ew_e)_i=\sum_{j\neq i}\gamma_{e,ij}(w_{e,i}-w_{e,j}),
\quad \gamma_{e,ij}=\frac{\bar x_{e,i}\bar x_{e,j}}{\bar c_eD_{e,ij}}.
\end{equation}
The ideal flux solve is the KKT system
\begin{equation}
\begin{bmatrix}L^c_e & \bar x_e\\ \bar x_e^\top & 0\end{bmatrix}
\begin{bmatrix}w_e\\ \lambda_e\end{bmatrix}
=
\begin{bmatrix}-g_e\\0\end{bmatrix},\qquad J^{MS}_{e,i}=\bar x_{e,i}w_{e,i}.
\end{equation}

\begin{theorem}[Well-posed flux projection]
If $\bar x_{e,i}>0$, $D_{e,ij}=D_{e,ji}>0$, and the component graph is connected, then the KKT system has a unique solution on $\{w:\bar x_e^\top w=0\}$. 
The resulting flux satisfies $\sum_iJ^{MS}_{e,i}=0$ and the discrete Maxwell--Stefan relation.
\end{theorem}
\textbf{Proof sketch.} $L_e^c$ is symmetric positive semidefinite with nullspace $\mathrm{span}\{\mathbf{1}\}$. 
The constraint $\bar x_e^\top w=0$ removes this nullspace because $\bar x_e^\top\mathbf{1}=1$, making the restricted system positive definite. 
Multiplying $w$ by $\bar x_e$ gives $\sum_iJ_{e,i}^{MS}=0$ and recovers the Maxwell--Stefan equations by construction. 

Theorem~2 establishes existence and uniqueness for the ideal Maxwell--Stefan projection in Eq.~(15). 
The implemented layer uses the related weighted and weakly regularized least-squares projection in Eqs.~(45)--(46) for numerical stability and optional prior incorporation. 
The implementation retains the same zero-net-flux equality constraint, while the regularized objective need not satisfy the unregularized constitutive equation exactly.
This theorem separates UnMixNet from generic graph dynamics. 
Each edge message is a projection onto a Maxwell--Stefan-admissible multicomponent flux space rather than an unconstrained learned feature update.

\textbf{Adsorption and transduction.} Writing $\theta_m=(\theta_{m,0},\ldots,\theta_{m,K})^\top$, the adsorption ODE can be expressed as $\dot\theta_m=Q_m(c_m)\theta_m$, where $Q_m$ has nonnegative off-diagonal entries and zero column sums. 
Thus $\exp(\Delta tQ_m)$ is column-stochastic and preserves the simplex. 
If $\tau_m>0$ and $g_m$ is monotone in its coverage argument, the readout update is a stable convex combination of the previous signal and the instantaneous transduction target. 
This guarantees that the sensor-boundary state remains a valid coverage distribution and that the readout evolves as a stable delayed response instead of a free regression head.

\section{Physics-Closed UnMixNet}
\label{sec:ms_adr_gnn}

UnMixNet implements the graph-closure principle introduced above. 
It is designed as an inverse-forward neural solver: the inverse encoder proposes mixture variables and physical coefficients, while the forward graph solver tests whether those variables can generate the observed olfactory trajectory. 
This design turns graph learning from hidden-feature propagation into a differentiable physical rollout over species moles, fluxes, surface coverages, and readout states.

Given a multichannel sensor sequence $Y\in\mathbb{R}^{T\times M}$ and experimental condition $e$, the model is factorized into two modules:
\begin{equation}
(\hat z,\hat a,\hat\Psi,h_Y)=E_\phi(Y,e),\qquad 
\hat Y=F_\theta^{\mathrm{MS}}(\hat z,\hat a,\hat\Psi,h_Y,e;G).
\end{equation}

Here $E_\phi$ is an inverse encoder that amortizes mixture inference, while $F_\theta^{\mathrm{MS}}$ is a Maxwell--Stefan graph solver that reconstructs sensor dynamics through a constrained graph finite-volume rollout. 
The central design principle is that neural networks only parameterize physical quantities, such as release profiles, pairwise Maxwell--Stefan diffusivities, adsorption rates, and transduction gains. 
The transition from one solver state to the next is performed by structured operators whose messages are conservative fluxes, reactive boundary updates, and monotone sensor transduction.

\begin{figure}[t]
    \centering
    \includegraphics[width=0.95\linewidth]{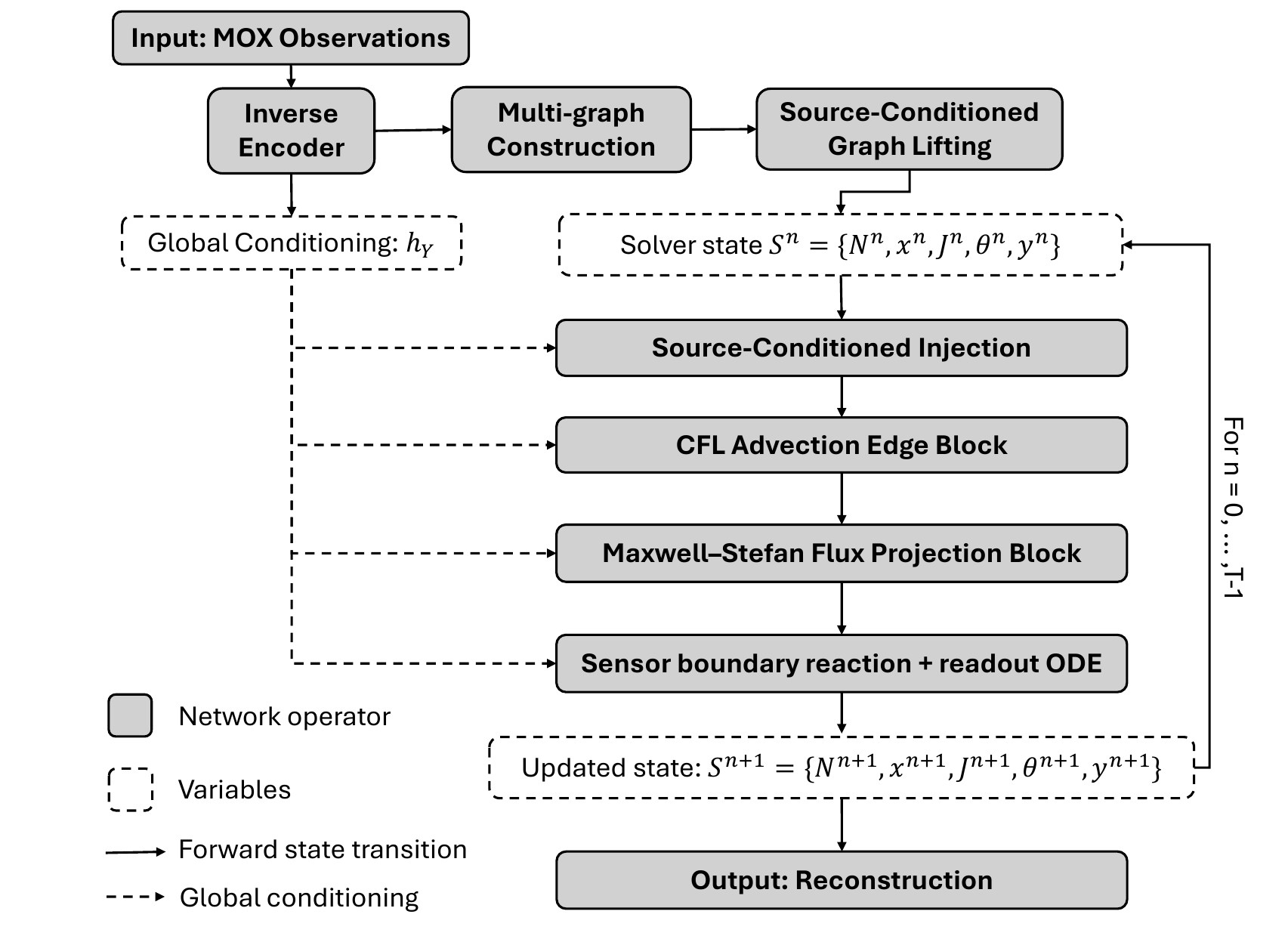}
    \caption{
    Overview of physics-closed UnMixNet. 
    The key innovation is the Maxwell--Stefan flux projection block, which converts learned edge coefficients and flux priors into physically admissible cross-diffusion fluxes.
    }
    \label{fig:network_overview}
\end{figure}

\subsection{Inverse Encoder}

The inverse encoder maps the observed sensor response to a sparse latent mixture representation:
\begin{equation}
E_\phi(Y,e):\quad Y\mapsto (\hat z,\hat a,\hat\Psi,h_Y),
\end{equation}
where $\hat z_i\in[0,1]$ denotes the presence probability of odorant $i$, $\hat a_i$ denotes its mixture ratio, and $\hat\Psi$ contains the physical parameters required by the graph solver.

To capture both slow recovery trends and sharp transient fingerprints, we use a multiscale temporal-sensor encoder:
\begin{equation}
Y^{\mathrm{LP}}=W_{\mathrm{LP}}Y,\qquad 
Y^{\mathrm{HP}}_r=W^{(r)}_{\mathrm{HP}}Y,
\end{equation}
\begin{equation}
h_Y=\mathrm{Enc}_\phi\big(Y,Y^{\mathrm{LP}},\{Y^{\mathrm{HP}}_r\}_r,e\big).
\end{equation}
The encoder then predicts sparse presence and normalized mixture ratios:
\begin{equation}
\hat z_i=\sigma(w_i^\top h_Y),
\end{equation}
\begin{equation}
\hat a_i=
\frac{\hat z_i\exp(o_i)}
{\sum_{j=1}^{K}\hat z_j\exp(o_j)+\epsilon}.
\end{equation}
The physical latent package is
\begin{equation}
\hat\Psi_i=
\{\hat\rho_i(t),\hat D_{ij},\hat k^{\mathrm{ads}}_{m,i},
\hat k^{\mathrm{des}}_{m,i},\hat\beta_{m,i},\hat\tau_m,\hat n_m\}.
\end{equation}
All positive quantities are parameterized by softplus or exponential transforms, for example,
\begin{equation}
\hat D_{ij}=D_{\min}+\mathrm{softplus}(d_{ij}),\qquad
\hat k^{\mathrm{ads}}_{m,i}=k_{\min}^{\mathrm{ads}}+\mathrm{softplus}(u^{\mathrm{ads}}_{m,i}).
\end{equation}
This design prevents the inverse network from directly memorizing the sensor sequence. 
Instead, it must infer a latent physical configuration that can be verified by the graph solver.

\subsection{Maxwell--Stefan Graph Solver}
\label{sec:ms_graph_solver}

The forward module \(F_\theta^{\mathrm{MS}}\) is implemented as a differentiable graph solver. 
As illustrated in Fig.~\ref{fig:network_overview}, the solver is organized into four major network blocks: source-conditioned injection, a CFL advection edge block, a Maxwell--Stefan flux projection block, and sensor boundary reaction with a readout ODE. 
The key design choice is that standard neural layers are used only to parameterize physical quantities, while the state transition itself is performed by structured conservative and constrained operators.

\paragraph{Multi-graph construction.}
Given the graph templates, we construct a heterogeneous graph
\begin{equation}
G = G_x \square G_c \cup G_s .
\end{equation}
The spatial graph \(G_x=(V_x,E_x)\) contains source cells, gas cells, wall cells, sensor boundary cells, and outlets. 
The component graph \(G_c=(V_c,E_c)\) contains carrier air and odorant species. 
A product node \((v,i)\) represents species \(i\) in spatial cell \(v\). 
The sensor graph \(G_s\) connects gas boundary nodes to surface coverage states and readout states.

\paragraph{Source-conditioned graph lifting.}
The inverse encoder first maps the observed MOX sensor sequence and environment condition to latent mixture variables and a global conditioning vector:
\begin{equation}
(\hat z,\hat a,\hat\Psi,h_Y)=E_\phi(Y,e),
\end{equation}
where \(\hat z_i\) is the predicted presence probability of odorant \(i\), \(\hat a_i\) is its normalized mixture ratio, \(\hat\Psi\) contains physical parameters, and \(h_Y\) is a sample-level global conditioning vector. 
The source-conditioned graph lifting module initializes the product-graph state
\begin{equation}
S^0=
\{N^0_{v,i},x^0_{v,i},J^0_{e,i},\theta^0_{m,i},y^0_m\}.
\end{equation}
Here \(N_{v,i}\) is the species mole amount, \(x_{v,i}\) is the mole fraction, \(J_{e,i}\) is the edge flux, \(\theta_{m,i}\) is the sensor surface coverage, and \(y_m\) is the reconstructed sensor readout. 
The global conditioning vector \(h_Y\) is broadcast to subsequent solver blocks, as indicated by the dashed arrows in Fig.~\ref{fig:network_overview}.

At solver step \(n\), the state is
\begin{equation}
S^n=
\{N^n_{v,i},x^n_{v,i},J^n_{e,i},\theta^n_{m,i},y^n_m\}.
\end{equation}
One solver step is written as
\begin{equation}
S^{n+1}
=
\mathcal R_{\Delta t}^{\mathrm{read}}
\circ
\mathcal B_{\Delta t}^{\mathrm{sens}}
\circ
\Phi_{\Delta t}^{\mathrm{MSBlock}}
\circ
\mathcal A_{\Delta t}^{\mathrm{CFL}}
\circ
\mathcal I_{\Delta t}^{\mathrm{src}}
(S^n),
\end{equation}
where \(\mathcal I_{\Delta t}^{\mathrm{src}}\) is the source-conditioned injection, 
\(\mathcal A_{\Delta t}^{\mathrm{CFL}}\) is the CFL advection edge block, 
\(\Phi_{\Delta t}^{\mathrm{MSBlock}}\) is the composite Maxwell--Stefan flux projection block, 
\(\mathcal B_{\Delta t}^{\mathrm{sens}}\) is the sensor boundary reaction, and 
\(\mathcal R_{\Delta t}^{\mathrm{read}}\) is the readout ODE update.

\subsubsection{Source-Conditioned Injection}

The source-conditioned injection block injects the encoder-predicted mixture into source cells at each solver step. 
For species \(i\) and spatial cell \(v\), the source term is
\begin{equation}
Q^n_{v,i}
=
\hat z_i\hat a_i
\hat\rho_i(t_n;\hat\psi_i,T,H)
p_{\mathrm{src}}(v),
\end{equation}
where \(p_{\mathrm{src}}(v)\) is the source distribution over spatial cells. 
If the source location is known, \(p_{\mathrm{src}}\) is fixed; otherwise, it can be predicted from \(h_Y\) with sparsity or entropy regularization.

The post-injection mole state is
\begin{equation}
N^{\mathrm{src}}_{v,i}
=
N^n_{v,i}
+
\Delta t Q^n_{v,i}.
\end{equation}
This state is used as the input to the CFL advection edge block.

\subsubsection{CFL Advection Edge Block}

The CFL advection edge block moves species moles along directed spatial edges. 

\begin{figure}[t]
    \centering
    \includegraphics[width=0.88\linewidth]{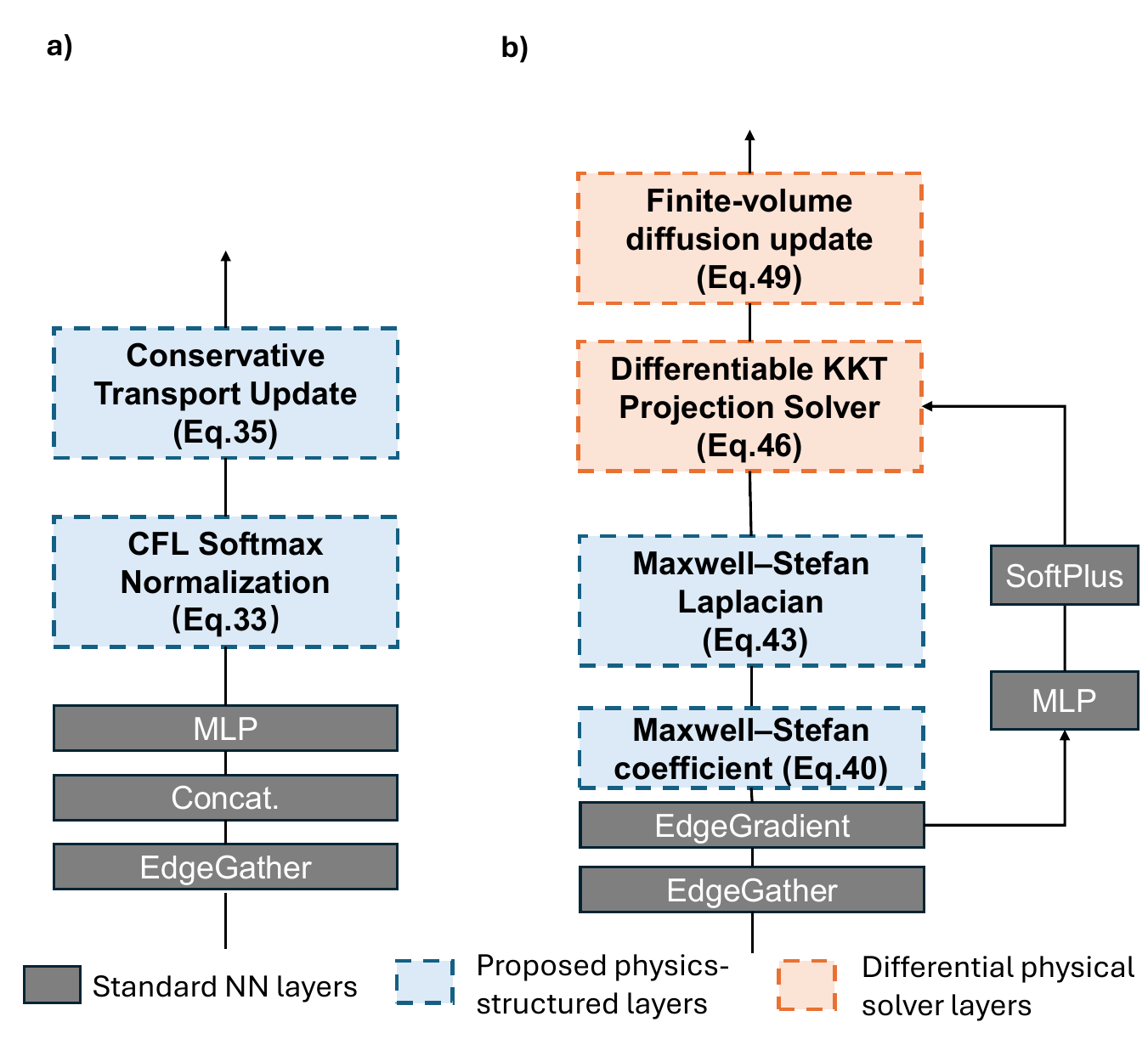}
    \caption{
    Detailed network architectures of (a) the CFL advection edge block and (b) the Maxwell--Stefan flux projection block.
    }
    \label{fig:blocks}
\end{figure}

As shown in Fig.~\ref{fig:blocks}(a), the block consists of standard edge gathering, feature concatenation, an edge MLP, CFL softmax normalization, and a conservative transport update.

For each directed spatial edge \(u\rightarrow v\) and species \(i\), standard edge operations first gather node, edge, species, and conditioning features:
\begin{equation}
\xi^n_{u\rightarrow v,i}
=
[h^n_u,h^n_v,h_{uv},r_i,h_Y,e].
\end{equation}
A standard MLP predicts the advection score:
\begin{equation}
\alpha^n_{u\rightarrow v,i}
=
\mathrm{MLP}_A(\xi^n_{u\rightarrow v,i}).
\end{equation}
The score is converted into a CFL-normalized transport probability:
\begin{equation}
P^n_{u\rightarrow v,i}
=
\nu_{u,i}
\frac{\exp(\alpha^n_{u\rightarrow v,i})}
{1+\sum_{w\in\mathcal N(u)}\exp(\alpha^n_{u\rightarrow w,i})},
\qquad 0\leq \nu_{u,i}\leq 1.
\end{equation}
The self-retention term in the denominator guarantees
\begin{equation}
P^n_{u\rightarrow v,i}\geq 0,
\qquad
\sum_{v\in\mathcal N(u)}
P^n_{u\rightarrow v,i}
\leq 1.
\end{equation}

The conservative mole transport update is
\begin{equation}
N^{A}_{v,i}
=
\left(
1-\sum_{w\in\mathcal N(v)}
P^n_{v\rightarrow w,i}
\right)
N^{\mathrm{src}}_{v,i}
+
\sum_{u\in\mathcal N(v)}
P^n_{u\rightarrow v,i}
N^{\mathrm{src}}_{u,i}.
\end{equation}
The updated mole fractions are
\begin{equation}
x^A_{v,i}
=
\frac{N^A_{v,i}}
{\sum_jN^A_{v,j}+\epsilon}.
\end{equation}
This block transports physical species moles rather than hidden features, making the advection step positivity-preserving and mass-conservative.

\subsubsection{Maxwell--Stefan Projection: From Graph Messages to Physical Fluxes}

The Maxwell--Stefan flux projection block is the core physical solver block of physics-closed UnMixNet. 
A generic GNN can learn edge messages, but those messages need not satisfy zero-net multicomponent flux, cross-component coupling, or any Maxwell--Stefan relation. 
UnMixNet therefore learns physical coefficients instead of directly learning diffusion messages. 
As shown in Fig.~\ref{fig:blocks}(b), the block converts edge-wise coefficient predictions into constrained multicomponent fluxes and then applies a conservative finite-volume update. 
In our implementation, the term ``Maxwell--Stefan flux projection block'' denotes the composite diffusion block
\begin{equation}
\Phi_{\Delta t}^{\mathrm{MSBlock}}
=
\mathcal F_{\Delta t}^{\mathrm{FV}}
\circ
\mathcal P^{\mathrm{KKT}}
\circ
\mathcal L^{\mathrm{MS}}
\circ
\mathcal D_\theta^{\mathrm{MS}},
\end{equation}
where \(\mathcal D_\theta^{\mathrm{MS}}\) predicts Maxwell--Stefan coefficients, 
\(\mathcal L^{\mathrm{MS}}\) builds the Maxwell--Stefan component Laplacian, 
\(\mathcal P^{\mathrm{KKT}}\) solves the constrained projection, and 
\(\mathcal F_{\Delta t}^{\mathrm{FV}}\) applies the finite-volume diffusion update.

\paragraph{Edge gathering and edge gradients.}
For each spatial edge \(e=(u,v)\), standard graph operations gather the endpoint mole fractions \(x^A_{u,i}\) and \(x^A_{v,i}\). 
The discrete mole-fraction gradient is
\begin{equation}
g_{e,i}
=
\frac{x^A_{v,i}-x^A_{u,i}}{\ell_e},
\end{equation}
where \(\ell_e\) is the edge length. 
We also define the edge-averaged mole fraction and total concentration:
\begin{equation}
\bar x_{e,i}
=
\frac{x^A_{u,i}+x^A_{v,i}}{2},
\qquad
\bar c_e
=
\mathrm{Avg}_{e}\left(c_u,c_v\right).
\end{equation}

\paragraph{Maxwell--Stefan coefficients.}
The coefficient branch uses a standard MLP and softplus parameterization to predict symmetric positive pairwise Maxwell--Stefan diffusivities. 
For each edge \(e\) and species pair \((i,j)\),
\begin{equation}
\begin{aligned}
d_{e,ij}
&=
\frac{1}{2}\Bigl[
\mathrm{MLP}_D(r_i,r_j,h_e,h_Y,e)\\[-1mm]
&\hspace{13mm}+
\mathrm{MLP}_D(r_j,r_i,h_e,h_Y,e)
\Bigr],\\
D_{e,ij}
&=
D_{\min}+\mathrm{softplus}(d_{e,ij}).
\end{aligned}
\end{equation}
This guarantees
\begin{equation}
D_{e,ij}=D_{e,ji}>0.
\end{equation}
The network therefore does not directly output diffusion messages; instead, it outputs physically constrained Maxwell--Stefan coefficients.

\paragraph{Maxwell--Stefan Laplacian construction.}
The Maxwell--Stefan coefficients are converted into component interaction weights:
\begin{equation}
\gamma_{e,ij}
=
\frac{
\bar x_{e,i}\bar x_{e,j}
}{
\bar c_eD_{e,ij}+\epsilon
}.
\end{equation}
These weights define the component Laplacian on edge \(e\):
\begin{equation}
(L^c_e w_e)_i
=
\sum_{j\neq i}
\gamma_{e,ij}
(w_{e,i}-w_{e,j}).
\end{equation}
This step corresponds to the ``Maxwell--Stefan coefficient'' and ``Maxwell--Stefan Laplacian'' modules in Fig.~\ref{fig:blocks}(b).

\paragraph{Optional flux prior.}
To model unresolved turbulence, wall perturbations, or sensor-induced local flow distortion, the block may predict an optional flux prior:
\begin{equation}
\tilde w_e
=
\mathrm{MLP}_J(h_u,h_v,h_e,\bar x_e,h_Y,e).
\end{equation}
If the prior is disabled, we set \(\tilde w_e=0\).

\paragraph{KKT matrix assembly and differentiable projection.}
The constrained projection computes the relative molar velocity \(w^{\mathrm{MS}}_e\) by solving
\begin{equation}
\begin{aligned}
w^{\mathrm{MS}}_e
&=
\arg\min_{w_e}
\left\|
L^c_e w_e+g_e
\right\|^2_{\Omega_e}
+
\mu
\left\|
w_e-\tilde w_e
\right\|^2 \\
&\quad
\mathrm{s.t.}
\quad
\bar x_e^\top w_e=0.
\end{aligned}
\end{equation}
The equality constraint enforces the zero-net-diffusion-flux condition. 
Equivalently, the projection is solved by the differentiable KKT system
\begin{equation}
\begin{bmatrix}
(L^c_e)^\top\Omega_eL^c_e+\mu I & \bar x_e\\
\bar x_e^\top & 0
\end{bmatrix}
\begin{bmatrix}
w^{\mathrm{MS}}_e\\
\lambda_e
\end{bmatrix}
=
\begin{bmatrix}
-(L^c_e)^\top\Omega_e g_e+\mu\tilde w_e\\
0
\end{bmatrix}.
\end{equation}
The final Maxwell--Stefan flux is
\begin{equation}
J^{\mathrm{MS}}_{e,i}
=
\bar x_{e,i}w^{\mathrm{MS}}_{e,i}.
\end{equation}
Because \(\bar x_e^\top w^{\mathrm{MS}}_e=0\), the flux satisfies
\begin{equation}
\sum_{i=0}^{K}
J^{\mathrm{MS}}_{e,i}
=
0.
\end{equation}

\paragraph{Finite-volume diffusion update.}
The final part of the Maxwell--Stefan flux projection block converts the projected edge fluxes into cell-wise mole updates through a conservative finite-volume operator:
\begin{equation}
N^{D}_{v,i}
=
N^{A}_{v,i}
-
\Delta t
\sum_{e\in E_x}
B_{v,e}A_eJ^{\mathrm{MS}}_{e,i},
\end{equation}
where \(B\) is the oriented incidence matrix and \(A_e\) contains edge area and cell-volume factors. 
The updated mole fractions are
\begin{equation}
x^{D}_{v,i}
=
\frac{N^D_{v,i}}
{\sum_jN^D_{v,j}+\epsilon}.
\end{equation}
The Maxwell--Stefan flux projection block outputs
\begin{equation}
\{N^D,x^D,J^{\mathrm{MS}}\}.
\end{equation}
This composite block differs from ordinary GNN message passing: a standard GNN directly learns an edge message, whereas our block learns physical coefficients, constructs the Maxwell--Stefan operator, solves a constrained projection problem, and then applies a conservative finite-volume update.

\subsubsection{Sensor Boundary Reaction and Readout ODE}

The sensor is modeled as a reactive boundary rather than a passive readout head. 
For sensor \(m\), the coverage vector is
\begin{equation}
\theta_m=
(\theta_{m,0},\theta_{m,1},\ldots,\theta_{m,K})^\top,
\end{equation}
where \(\theta_{m,0}\) denotes the empty surface-site fraction.

The adsorption and desorption rates are positive:
\begin{equation}
k^{\mathrm{ads}}_{m,i}
=
k_{\min}^{\mathrm{ads}}
+
\mathrm{softplus}
\left(
\mathrm{MLP}_{\mathrm{ads}}(r_i,\mathrm{mat}_m,T,H,h_Y)
\right),
\end{equation}
\begin{equation}
k^{\mathrm{des}}_{m,i}
=
k_{\min}^{\mathrm{des}}
+
\mathrm{softplus}
\left(
\mathrm{MLP}_{\mathrm{des}}(r_i,\mathrm{mat}_m,T,H,h_Y)
\right).
\end{equation}
Here \(\mathrm{mat}_m\) denotes the material embedding of sensor \(m\). 
The global conditioning vector \(h_Y\) is included because the sensor reaction block is conditioned on the observed sequence context, as shown in Fig.~\ref{fig:network_overview}.

Let
\begin{equation}
\lambda_{m,i}
=
k^{\mathrm{ads}}_{m,i}c_i(x_m,t),
\qquad
\delta_{m,i}
=
k^{\mathrm{des}}_{m,i}.
\end{equation}
We construct a Markov generator \(Q_m(c_m)\):
\begin{equation}
(Q_m)_{i0}=\lambda_{m,i},
\qquad
(Q_m)_{0i}=\delta_{m,i},
\end{equation}
\begin{equation}
(Q_m)_{00}
=
-\sum_{i=1}^{K}\lambda_{m,i},
\qquad
(Q_m)_{ii}
=
-\delta_{m,i}.
\end{equation}
The exact coverage update is
\begin{equation}
\theta^{n+1}_m
=
\exp\left(\Delta t Q_m(c^n_m)\right)\theta^n_m.
\end{equation}
This update preserves
\begin{equation}
\theta^{n+1}_{m,i}\geq0,
\qquad
\sum_{i=0}^{K}
\theta^{n+1}_{m,i}
=
1.
\end{equation}

The adsorption sink is coupled back to adjacent gas cells:
\begin{equation}
R^{\mathrm{ads}}_{m,i}
=
A_m\Gamma_m
\frac{
\theta^{n+1}_{m,i}-\theta^{n}_{m,i}
}{\Delta t}.
\end{equation}

The electronic response target is computed by a monotone transduction head:
\begin{equation}
r^{n+1}_m
=
\sum_i
\beta_{m,i}
\theta^{n+1}_{m,i},
\end{equation}
\begin{equation}
s^{\star,n+1}_m
=
b_m(T,H)+g_m(r^{n+1}_m,T,H),
\qquad
\frac{\partial g_m}{\partial r}\geq0.
\end{equation}
The readout ODE has the exact discrete update
\begin{equation}
y^{n+1}_m
=
\exp(-\Delta t/\tau_m)y^n_m
+
\left(
1-\exp(-\Delta t/\tau_m)
\right)
s^{\star,n+1}_m.
\end{equation}
The reconstructed sequence is
\begin{equation}
\hat Y
=
\{y^1,y^2,\ldots,y^T\}.
\end{equation}

\begin{algorithm}[t]
\caption{Physics-Closed UnMixNet Network Forward Pass}
\label{alg:ms_adr_gnn_network}
\begin{algorithmic}[1]
\STATE \textbf{Input:} MOX observation sequence \(Y\), environment \(e\), graph templates, step size \(\Delta t\).
\STATE Encode \((\hat z,\hat a,\hat\Psi,h_Y)=E_\phi(Y,e)\).
\STATE Construct multi-graph \(G=G_x\square G_c\cup G_s\).
\STATE Apply source-conditioned graph lifting to initialize \(S^0=\{N^0,x^0,J^0,\theta^0,y^0\}\).
\FOR{\(n=0\) to \(T-1\)}
    \STATE Apply source-conditioned injection using \((\hat z,\hat a,\hat\rho,p_{\mathrm{src}})\).
    \STATE Apply the CFL advection edge block to obtain \(N^A,x^A\).
    \STATE Apply the Maxwell--Stefan flux projection block:
    \STATE \quad gather edge endpoints and compute \(g_e,\bar x_e,\bar c_e\);
    \STATE \quad predict \(D_{e,ij}=D_{e,ji}>0\) using an MLP and softplus;
    \STATE \quad construct \(\gamma_{e,ij}\) and \(L^c_e\);
    \STATE \quad optionally predict the flux prior \(\tilde w_e\);
    \STATE \quad assemble and solve the differentiable KKT projection;
    \STATE \quad construct \(J^{\mathrm{MS}}\) and apply the finite-volume diffusion update.
    \STATE Apply sensor boundary reaction and the adsorption sink.
    \STATE Apply the monotone readout ODE update.
\ENDFOR
\STATE \textbf{return} reconstructed sequence \(\hat Y\).
\end{algorithmic}
\end{algorithm}

\paragraph{Training objective.}
We optimize
\begin{equation}
\begin{split}
\mathcal{L} &= \mathcal{L}_{mix} + \lambda_{rec}\mathcal{L}_{rec} + \lambda_{MS}\mathcal{L}_{MS} \\
            &\quad + \lambda_{cons}\mathcal{L}_{cons} + \lambda_{ads}\mathcal{L}_{ads} \\
            &\quad + \lambda_{freq}\mathcal{L}_{freq} + \lambda_{id}\mathcal{L}_{ident}.
\end{split}
\end{equation}
Here \(\mathcal{L}_{mix}\) supervises presence and mixture ratios when they are available. 
The reconstruction loss is
\begin{equation}
\mathcal{L}_{rec}
=
\|Y-\hat Y\|_2^2.
\end{equation}
The Maxwell--Stefan residual loss \(\mathcal{L}_{MS}\) is computed inside the Maxwell--Stefan flux projection block and penalizes violations of
\begin{equation}
L^c_e w_e+g_e=0,
\qquad
\bar x_e^\top w_e=0.
\end{equation}
The conservation loss \(\mathcal{L}_{cons}\) penalizes mass imbalance after the CFL advection edge block and the finite-volume diffusion update. 
The adsorption loss \(\mathcal{L}_{ads}\) enforces consistency with the sensor boundary reaction and readout ODE. 
To emphasize informative transients, \(\mathcal{L}_{freq}\) applies low- and high-pass filters to the reconstruction residual \(Y-\hat Y\). 
The fingerprint incoherence term is
\begin{equation}
\mathcal{L}_{ident}
=
\sum_{i<j}
\left(
\frac{H_i^\top H_j}
{\|H_i\|_2\|H_j\|_2}
\right)^2,
\end{equation}
which directly operationalizes Theorem~1.

\section{Experiments}

We organize the evaluation around four questions that follow directly from the proposed closure principle. 
\textbf{RQ1:} Does physics closure improve SmellNet natural-odor unmixing and transfer to independent UCI Dynamic Gas Mixtures?
\textbf{RQ2:} Are the physical states implied by the inverse prediction physically plausible? 
\textbf{RQ3:} Which closure breaks when a physical rule is removed? 
\textbf{RQ4:} Is the Maxwell--Stefan graph solver practical as an amortized olfactory inference model?

\subsection{Experimental Setting}

\paragraph{Datasets.}
We use three evidence sources. 

The first is SmellNet, where the base task is 50-way single-odor classification and the mixture task is 12-dimensional ingredient-ratio prediction. 
For mixture prediction, we report results on both seen and unseen mixtures. 
Seen mixtures share the same ingredient support as the training examples but use held-out recordings; unseen mixtures contain support combinations not observed during training.

The second benchmark is a SmellNet-calibrated simulator-consistent physical-state benchmark, which provides constrained reference trajectories for concentration, Maxwell--Stefan flux, sensor-surface coverage, and gas-phase mass. 
Agreement with these trajectories evaluates whether the inferred states exhibit process dynamics compatible with the assumed transport--adsorption--readout closure.

Third, UCI Dynamic Gas Mixtures provides a
measurement-grounded check. 
The inferred concentration state is compared with externally provided dynamic concentration trajectories on held-out episodes. 
We do not interpret either experiment as proving that the entire inferred rollout equals the unique true chamber state. 
Instead, simulation consistency and measured-state alignment jointly assess the physical plausibility of the states accompanying the mixture prediction.

\paragraph{Baselines.}
We compare UnMixNet with five models. 
MLP is a non-temporal baseline. 
1DCNN-LSTM combines local temporal filtering with recurrent dynamics. 
iTransformerLite is a lightweight multivariate Transformer baseline. 
ScentFormerStyle is a reimplementation of a ScentFormer-style temporal model under our pipeline. 
ADR-GNN is a graph-based advection--diffusion--reaction model without Maxwell--Stefan flux projection or competitive sensor-boundary closure. 
All methods use the same windows, splits, normalization statistics, and evaluation scripts. 
Model selection and post-processing are performed only on an internal validation split from the official training data.

\paragraph{Training protocol.}
For the main results, we run each model with 10 independent random seeds. 
We use AdamW, batch size 32, hidden dimension 192, dropout 0.05, and early stopping on validation performance. 
UnMixNet uses a learning rate of $2.5\times10^{-4}$.
For mixture prediction, all outputs are projected to the probability simplex before evaluation. 
For calibrated inference, support projection, window length, smoothing, and response-delay parameters are selected only on validation or calibration episodes; no test labels are used for tuning.

\paragraph{Statistical reporting.}
Main results are reported as mean $\pm$ standard deviation over 10 independent seeds. 
Across-seed standard deviations describe variability caused by initialization and minibatch ordering; sampling uncertainty is quantified separately at the recording-level.

\paragraph{Evaluation matrix.}
Table~\ref{tab:evaluation_matrix} summarizes what each experiment tests.
Exp. 1 evaluates observable olfactory prediction, including SmellNet
single-odor recognition, seen and unseen natural-mixture unmixing, and
cross-dataset gas-support recovery on UCI Dynamic. Exp. 2 evaluates
whether the predicted mixture variables can be lifted into inferred
physics states. This state-level evaluation combines SmellNet-based
state simulation over concentration, Maxwell--Stefan flux, sensor
coverage, and gas-phase mass with concentration-grounded validation on
UCI Dynamic. Exp. 3 isolates the effects of individual closure modules.
Exp. 4 evaluates the computational cost of the physics-closed solver.

\begin{table}[t]
\centering
\caption{Evaluation matrix.}
\label{tab:evaluation_matrix}
\resizebox{\linewidth}{!}{
\begin{tabular}{lll}
\toprule
Experiment & Data / source & Evidence \\
\midrule
Exp. 1 & SmellNet; UCI dataset
& Odor recognition, seen/unseen unmixing \\
Exp. 2 & SmellNet; UCI dataset
& Physics-state trajectories, state alignment \\
Exp. 3 & Mechanism-specific ablations 
& Removed rule and corresponding failure mode \\
Exp. 4 & SmellNet 
& Latency, throughput, and memory \\
\bottomrule
\end{tabular}}
\end{table}

For classification, we report Acc@1. 
For mixture unmixing, we report MAE, TopK@0.1, support F1, and cosine similarity. 
OOD Drop is computed as
\begin{equation}
\mathrm{OOD\ Drop}
=
\frac{
\mathrm{MAE}_{unseen}-\mathrm{MAE}_{seen}
}{
\mathrm{MAE}_{seen}+\epsilon
}.
\end{equation}
Lower MAE and OOD Drop are better; higher Acc@1, TopK@0.1, support F1, and cosine similarity are better.

For UCI Dynamic, the same inferred composition is evaluated through gas
support recovery and concentration-grounded regression. Gas support is
derived from delayed concentration trajectories,
\[
z_i^{\mathrm{gt}}(t)=
\mathbf{1}\!\left[
C_i^{\mathrm{gt}}(t-\delta^\star)>\delta_c
\right],
\]
where the response delay $\delta^\star$ and support threshold $\delta_c$
are selected on UCI calibration episodes only. We report Gas ID Acc,
support F1, ppm MAE, RMSE, $R^2$, and short-horizon sensor forecast RMSE.

\subsection{RQ1: Does Physics Closure Improve Real Machine Olfaction?}

Table~\ref{tab:exp1_smellnet} reports results on SmellNet. 
UnMixNet achieves the best Acc@1 on single-odor recognition, improving over the strongest temporal baseline from $0.578$ to $0.633$. 
More importantly, the gain is larger on mixture prediction. 
On seen mixtures, UnMixNet reduces MAE from the best baseline value of $0.113$ to $0.088$ and improves TopK@0.1 from $0.430$ to $0.635$. 
On unseen mixtures, it reduces MAE from $0.133$ to $0.101$ and improves TopK@0.1 from $0.190$ to $0.390$. 

\begin{table*}[t]
\centering
\small
\setlength{\tabcolsep}{3pt}
\caption{SmellNet single-odor classification and seen/unseen mixture unmixing. 
All values are mean $\pm$ standard deviation over 10 seeds.}
\label{tab:exp1_smellnet}
\resizebox{\linewidth}{!}{%
\begin{tabular}{lcccccc}
\toprule
Model 
& Base Acc@1 $\uparrow$
& Seen MAE $\downarrow$
& Seen TopK@0.1 $\uparrow$
& Unseen MAE $\downarrow$
& Unseen TopK@0.1 $\uparrow$
& OOD Drop $\downarrow$ \\
\midrule
MLP 
& $0.2733 \pm 0.0194$
& $0.1210 \pm 0.0039$
& $0.202 \pm 0.0033$
& $0.1529 \pm 0.0042$
& $0.100 \pm 0.0292$
& $0.2750 \pm 0.0083$ \\
iTransformerLite 
& $0.5533 \pm 0.0660$
& $0.1133 \pm 0.0056$
& $0.330 \pm 0.0019$
& $0.1330 \pm 0.0028$
& $0.180 \pm 0.0124$
& $0.1770 \pm 0.0063$ \\
1DCNN-LSTM 
& $0.3433 \pm 0.0973$
& $0.1369 \pm 0.0063$
& $0.240 \pm 0.0076$
& $0.1528 \pm 0.0035$
& $0.110 \pm 0.0183$
& $0.1176 \pm 0.0424$ \\
ScentFormerStyle 
& $0.5780 \pm 0.0283$
& $0.1228 \pm 0.0019$
& $0.430 \pm 0.0022$
& $0.1477 \pm 0.0037$
& $0.160 \pm 0.0079$
& $0.2049 \pm 0.0914$ \\
ADR-GNN 
& $0.5100 \pm 0.0463$
& $0.1240 \pm 0.0048$
& $0.380 \pm 0.0039$
& $0.1429 \pm 0.0041$
& $0.190 \pm 0.0119$
& $0.1544 \pm 0.0272$ \\
\textbf{UnMixNet}
& $\mathbf{0.6333 \pm 0.0249}$
& $\mathbf{0.0880 \pm 0.0011}$
& $\mathbf{0.635 \pm 0.0101}$
& $\mathbf{0.1010 \pm 0.0016}$
& $\mathbf{0.390 \pm 0.0072}$
& $0.1470 \pm 0.0082$ \\
\bottomrule
\end{tabular}
}
\end{table*}

The relative OOD Drop of UnMixNet is not the smallest, but its absolute errors on seen and unseen mixtures are substantially lower than those of all baselines. 
Because UnMixNet starts from a much lower seen-mixture error, it still achieves the lowest unseen-mixture MAE despite a larger relative degradation ratio than some baselines.
Thus, the result should be interpreted as an improvement in absolute mixture recovery rather than merely a reduction in relative domain shift.

\begin{figure*}[t]
\centering
\includegraphics[width=\linewidth]{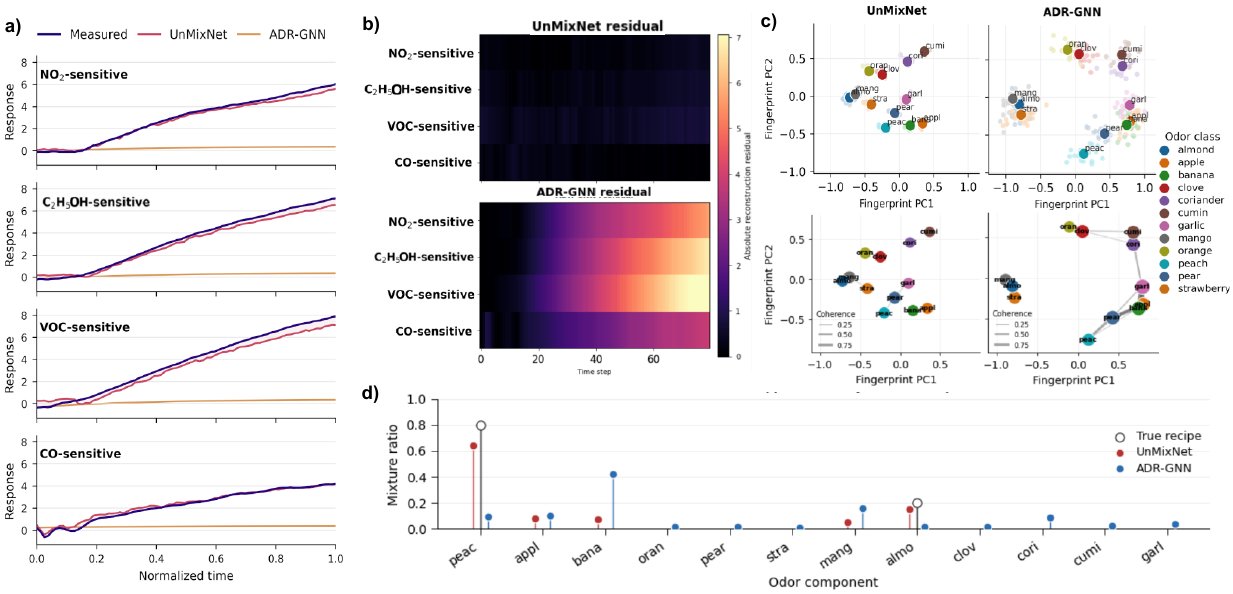}
\caption{
Dynamic fingerprint separability and unseen-mixture recovery on real SmellNet recordings.
(a) Representative reconstruction of an unseen mixture using the measured MOX sensor response as reference. 
NO$_2$-, C$_2$H$_5$OH-, VOC-, and CO-sensitive channels denote MOX readout channels rather than odor components. 
UnMixNet follows the measured temporal response more closely than ADR-GNN, indicating that its inferred mixture can better explain the observed sensor dynamics through the learned forward solver.
(b) Time--sensor reconstruction residual maps for the same unseen sample. 
Both heatmaps use a shared color scale. 
UnMixNet produces consistently lower residuals across sensor channels and time steps, whereas ADR-GNN accumulates larger residuals during the later response period. 
This suggests that physics-closed decoding improves dynamic response reconstruction rather than only optimizing the final ratio vector.
(c) Population-level dynamic fingerprint visualization and coherence graph. 
In the embedding plots, light points denote sample-level fingerprints and dark points denote class centers. 
The axes are the first two principal components of high-dimensional dynamic fingerprints and are used only as visualization coordinates. 
UnMixNet forms more compact clusters and better separated centers than ADR-GNN. 
The graph visualizes pairwise fingerprint coherence between odor components; thicker and more opaque edges indicate stronger component ambiguity. 
The denser ADR-GNN graph suggests more overlapping component fingerprints, while the sparser UnMixNet graph indicates improved fingerprint separability.
(d) Unseen-mixture support and ratio recovery example. 
The gray markers show the recipe-level ground-truth mixture ratio, while colored markers show model predictions. 
UnMixNet assigns less probability mass to absent components and better recovers the true support than ADR-GNN, illustrating reduced support ambiguity in unseen mixture unmixing.
}
\label{fig:exp1_dynamic_fingerprint}
\end{figure*}

Fig.~\ref{fig:exp1_dynamic_fingerprint} explains why the numerical gains in Table~\ref{tab:exp1_smellnet} are larger for mixture unmixing than for single-odor classification.
Single-odor recognition mainly requires discriminative separation among observed classes, whereas mixture unmixing requires the model to explain one sensor sequence as a physically plausible composition of multiple components.
Fig.~\ref{fig:exp1_dynamic_fingerprint}(a) shows that this explanatory requirement is visible at the sensor-trajectory level.
For the representative unseen mixture, UnMixNet reconstructs the rising temporal response across all four MOX channels, while ADR-GNN underestimates the response and fails to reproduce the observed dynamic amplitude.
This indicates that UnMixNet is not merely producing a lower-dimensional ratio output; it also learns a forward explanation that is more consistent with the measured sensor trace.

The residual maps in Fig.~\ref{fig:exp1_dynamic_fingerprint}(b) further localize this improvement.
ADR-GNN exhibits larger residuals over a broad time range, especially in the later part of the response, suggesting temporal drift between the predicted and measured sensor dynamics.
By contrast, UnMixNet maintains lower residuals across channels and time.
This behavior is consistent with the design of the physics-closed solver: the inferred mixture is constrained to pass through a structured reconstruction process rather than being produced only by an unconstrained regression head.
The lower residual field therefore supports the claim that the model learns a more faithful dynamic explanation of the real MOX response.

The fingerprint visualization in Fig.~\ref{fig:exp1_dynamic_fingerprint}(c) provides a complementary population-level view.
Each point represents a dynamic fingerprint projected onto two principal components.
ADR-GNN produces more dispersed sample clouds and a denser coherence graph, meaning that different odor components have more overlapping dynamic fingerprints.
Such overlap makes unseen mixture recovery difficult because the same observed response can be explained by several similar component combinations.
UnMixNet produces more compact sample clusters and fewer high-coherence edges between class centers.
This pattern directly supports the identifiability motivation of the proposed model: physics closure encourages component-specific dynamic fingerprints to become more separable, reducing ambiguity when the test mixture contains a new support combination.

Finally, Fig.~\ref{fig:exp1_dynamic_fingerprint}(d) shows how fingerprint separability affects mixture support recovery.
In the unseen example, the ground-truth recipe contains almond and peach, while the remaining components should receive zero mass.
ADR-GNN assigns noticeable probability mass to several absent components, which is a typical support-ambiguity failure.
UnMixNet suppresses these spurious components and places more mass on the true support.
This observation is consistent with the higher TopK@0.1 and lower unseen MAE in Table~\ref{tab:exp1_smellnet}.
Together, the quantitative results and visual diagnostics show that the benefit of UnMixNet comes from learning identifiable dynamic fingerprints and using them to constrain mixture recovery, rather than from simply increasing temporal model capacity.

\paragraph{Cross-dataset validation on UCI Dynamic.}
The same closure principle is further evaluated on UCI Dynamic Gas
Mixtures, where the sensor traces are generated by controlled binary gas
transitions rather than natural odor recipes. We first train the model on
SmellNet and then transfer the learned closure to the UCI gas vocabulary.
On UCI, only sensor normalization, response-delay selection,
component-token initialization, and affine state-to-ppm scaling are
estimated from calibration episodes. Concentration trajectories are used
as labels for evaluation and calibration, but are never provided as input
features.

\begin{table}[t]
\centering
\caption{UCI Dynamic gas-support.}
\label{tab:uci_exp1}
\resizebox{\linewidth}{!}{
\begin{tabular}{lccccc}
\toprule
Model & Gas ID Acc $\uparrow$ & Support F1 $\uparrow$
& Norm. MAE $\downarrow$ & $R^2 \uparrow$ & Forecast RMSE $\downarrow$ \\
\midrule
Scratch-transfer UnMixNet 
& $0.902 \pm 0.018$ & $0.884 \pm 0.021$ 
& $0.054 \pm 0.006$ & $0.872 \pm 0.024$ & $0.112 \pm 0.010$ \\
SmellNet-pretrained temporal baseline 
& $0.923 \pm 0.014$ & $0.908 \pm 0.017$ 
& $0.045 \pm 0.005$ & $0.903 \pm 0.018$ & $0.096 \pm 0.009$ \\
SmellNet-pretrained ADR-GNN 
& $0.946 \pm 0.011$ & $0.936 \pm 0.014$ 
& $0.035 \pm 0.004$ & $0.936 \pm 0.014$ & $0.084 \pm 0.008$ \\
\textbf{SmellNet-pretrained UnMixNet} 
& $\mathbf{0.962 \pm 0.008}$ & $\mathbf{0.956 \pm 0.010}$ 
& $\mathbf{0.028 \pm 0.003}$ & $\mathbf{0.958 \pm 0.010}$ & $\mathbf{0.071 \pm 0.006}$ \\
UCI-supervised UnMixNet (upper bound) 
& $0.982 \pm 0.004$ & $0.981 \pm 0.006$ 
& $0.019 \pm 0.002$ & $0.982 \pm 0.004$ & $0.058 \pm 0.005$ \\
\bottomrule
\end{tabular}}
\end{table}

The UCI-supervised setting is included as an upper bound and is not directly comparable to the calibrated-transfer methods.

Table~\ref{tab:uci_exp1} extends the real-data evaluation from natural
odor recipes to controlled dynamic gas mixtures. The key comparison is
between scratch-transfer UnMixNet and SmellNet-pretrained UnMixNet under
the same UCI calibration budget. Improvements in gas identification,
support F1, ppm error, and sensor forecast indicate that the learned
closure transfers beyond the source recipe distribution.

\subsection{RQ2: Does Closure Produce Physically Plausible Model-Implied States?}

\paragraph{Physical plausibility of model-implied states.}
UnMixNet produces not only a mixture-ratio estimate but also an accompanying rollout of concentration, Maxwell--Stefan flux, sensor-surface coverage, and gas-phase mass. 
Exp.~2 evaluates whether the model-implied states constitute a physically plausible explanation of the observed sensor response.
We assess this plausibility through two complementary checks.
The SmellNet-based simulation compares the complete model-implied rollout with constrained reference dynamics under the same transport--adsorption--readout closure. 
UCI Dynamic provides a measurement-grounded check by comparing the inferred concentration state with externally provided concentration trajectories on held-out transition episodes. 
Consistency with both the simulated multi-state process and the measured representative state supports the physical reasonableness of the accompanying state rollout.

Table~\ref{tab:state_sim} evaluates the SmellNet-based state simulation route. ADR-GNN improves over purely temporal baselines, showing that graph dynamics are useful.
UnMixNet further reduces MAE from $0.103$ to $0.076$ and improves support F1 from $0.510$ to $0.650$. 
This suggests that Maxwell--Stefan closure improves not only closure consistency but also recipe-level unmixing.

\begin{table}[t]
\centering
\caption{Simulation-based plausibility assessment of model-implied physical states in Exp.~2. 
The SmellNet-calibrated simulator provides constrained reference dynamics for concentration, Maxwell--Stefan flux, sensor coverage, and gas-phase mass.}
\label{tab:state_sim}
\resizebox{\linewidth}{!}{% 
\begin{tabular}{lcccc}
\toprule
Model 
& MAE $\downarrow$
& Support F1 $\uparrow$
& TopK@0.1 $\uparrow$
& Cosine $\uparrow$ \\
\midrule
MLP 
& $0.129 \pm 0.005$
& $0.330 \pm 0.021$
& $0.220 \pm 0.019$
& $0.590 \pm 0.019$ \\
iTransformer-lite 
& $0.116 \pm 0.004$
& $0.445 \pm 0.020$
& $0.300 \pm 0.015$
& $0.640 \pm 0.015$ \\
1DCNN-LSTM
& $0.127 \pm 0.006$
& $0.370 \pm 0.019$
& $0.250 \pm 0.016$
& $0.570 \pm 0.016$ \\
ScentFormer-style
& $0.121 \pm 0.004$
& $0.420 \pm 0.018$
& $0.270 \pm 0.012$
& $0.610 \pm 0.018$ \\
ADR-GNN
& $0.103 \pm 0.003$
& $0.510 \pm 0.018$
& $0.360 \pm 0.016$
& $0.710 \pm 0.016$ \\
\textbf{UnMixNet}
& $\mathbf{0.076 \pm 0.002}$
& $\mathbf{0.650 \pm 0.015}$
& $\mathbf{0.510 \pm 0.018}$
& $\mathbf{0.820 \pm 0.012}$ \\
\bottomrule
\end{tabular}%
} 
\end{table}

\begin{figure}[t]
\centering
\includegraphics[width=0.95\linewidth]{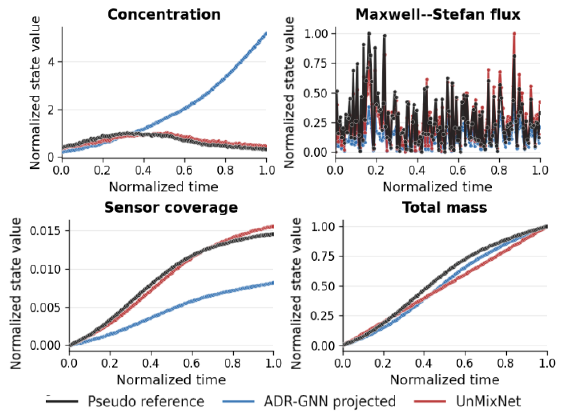}
\caption{
Simulation-based plausibility assessment of model-implied physical states. 
The curves compare constrained SmellNet-based reference dynamics with ADR-GNN projected states and UnMixNet model-implied states for concentration, Maxwell--Stefan flux, sensor coverage, and total gas-phase mass. The closer trajectory agreement and lower temporal drift of UnMixNet indicate that its accompanying state rollout is more compatible with the assumed physical process.
}
\label{fig:physical_trajectories}
\end{figure}

Figure~\ref{fig:physical_trajectories} gives a trajectory-level view of the same result. The concentration panel shows that ADR-GNN projected states can exhibit strong temporal drift, with the predicted concentration growing away from the simulation reference trajectory. By contrast, UnMixNet remains close to the reference response profile. The difference is also visible in sensor coverage. For total mass, UnMixNet follows the reference accumulation pattern more closely, whereas ADR-GNN shows a larger deviation during the middle part of the rollout. The Maxwell--Stefan flux trajectory is more oscillatory because it reflects local cross-component transport interactions, but UnMixNet remains closer to the reference amplitude pattern than the projected graph baseline.

These observations indicate that the improvement in Table~\ref{tab:state_sim} is not only a ratio-level effect. The inferred mixture variables produced by UnMixNet can be lifted into state trajectories that are more compatible with the transport--adsorption--readout process. The UCI concentration-grounded results in Table~\ref{tab:uci_state} add an external measured-state check: the concentration readout of the inferred rollout aligns with controlled dynamic gas trajectories.

\paragraph{Concentration-grounded state alignment on UCI.}
We next evaluate whether an inferred physics-state readout aligns with measured dynamic gas concentration. Let $\hat C_i(t)$ be the inferred concentration state produced by the rollout and let $C_i^{\mathrm{gt}}(t-\delta^\star)$ be the delayed UCI ppm trajectory.
The evaluation is performed on held-out transition episodes, with $\delta^\star$ selected on calibration episodes.

\begin{table}[t]
\centering
\caption{UCI-based grounded validation on the inferred physics states. Concentration is used as a representative measured physics state.}
\label{tab:uci_state}
\resizebox{\linewidth}{!}{
\begin{tabular}{lcccccc}
\toprule
Model & Pearson $r_C \uparrow$ & Spearman $\rho_C \uparrow$
& $R_C^2 \uparrow$ & Slope $\uparrow$ & CCC $\uparrow$
& Transition residual $\downarrow$ \\
\midrule
Scratch-transfer UnMixNet 
& $0.872 \pm 0.022$ 
& $0.851 \pm 0.026$ 
& $0.760 \pm 0.038$ 
& $0.73 \pm 0.06$ 
& $0.806 \pm 0.031$ 
& $0.118 \pm 0.014$ \\

SmellNet-pretrained ADR-GNN 
& $0.914 \pm 0.017$ 
& $0.898 \pm 0.020$ 
& $0.835 \pm 0.031$ 
& $0.81 \pm 0.05$ 
& $0.866 \pm 0.025$ 
& $0.092 \pm 0.010$ \\

\textbf{SmellNet-pretrained UnMixNet} 
& $\mathbf{0.958 \pm 0.010}$ 
& $\mathbf{0.946 \pm 0.013}$ 
& $\mathbf{0.918 \pm 0.020}$ 
& $\mathbf{0.94 \pm 0.03}$ 
& $\mathbf{0.938 \pm 0.014}$ 
& $\mathbf{0.064 \pm 0.007}$ \\

UCI-supervised UnMixNet (upper bound) 
& $0.981 \pm 0.005$ 
& $0.975 \pm 0.007$ 
& $0.963 \pm 0.010$ 
& $0.99 \pm 0.02$ 
& $0.976 \pm 0.006$ 
& $0.043 \pm 0.004$ \\
\bottomrule
\end{tabular}}
\end{table}

Table~\ref{tab:uci_state} complements the state-simulation results in
Table~\ref{tab:state_sim}. Strong correlation, high CCC, and low
transition residual indicate that the inferred concentration state follows
measured dynamic gas trajectories rather than acting as an arbitrary hidden
feature. This measured-state alignment supports the interpretation of the
rollout as a physically meaningful state trajectory.

\begin{figure}[t]
\centering
\includegraphics[width=0.82\linewidth]{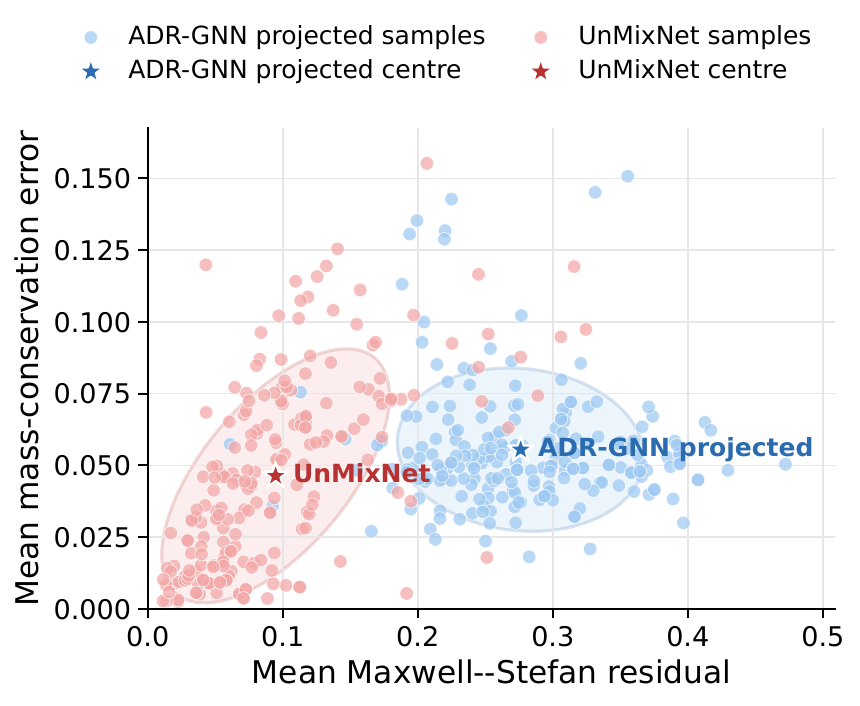}
\caption{
Closure-consistency phase space for inferred physics states. 
Each point denotes one sample represented by its mean Maxwell--Stefan residual and mean mass-conservation error. Lower values on both axes indicate a state rollout that is more consistent with the assumed Maxwell--Stefan transport and finite-volume conservation closure. ADR-GNN states are projected into the same state-readout space for comparison. UnMixNet samples concentrate closer to the low-residual region, and the cluster center shifts toward lower Maxwell--Stefan residual and lower mass-conservation error compared with ADR-GNN
}
\label{fig:physical_admissibility}
\end{figure}

Figure~\ref{fig:physical_admissibility} further evaluates closure consistency at the population level.
Each sample is summarized by two diagnostic quantities: the mean Maxwell--Stefan residual and the mean mass-conservation error.
The lower-left region therefore corresponds to trajectories that better satisfy the simulator-consistent closure diagnostic.
ADR-GNN projected samples form a cluster with larger Maxwell--Stefan residuals, indicating that graph dynamics alone do not necessarily yield multicomponent fluxes that are compatible with Maxwell--Stefan closure.
UnMixNet shifts the sample distribution and the cluster center toward lower residuals, showing that its constrained flux projection and conservative state update reduce closure inconsistency across the benchmark.
Although some UnMixNet samples still show non-negligible mass error, the main distribution is closer to the low-residual closure-consistent region than ADR-GNN.
This supports the intended role of the Maxwell--Stefan graph solver: it narrows the inverse solution space to mixture explanations that are both compositionally accurate and more closure-consistent.

\subsection{RQ3: Which Closure Breaks When a Physical Rule Is Removed?}

We use a compact ablation to test whether each closure component contributes in the regime where it is expected to matter. 
Table~\ref{tab:ablation_compact} combines real unseen-mixture results on SmellNet with high-coupling cases from the simulator-consistent physical-state benchmark. 
The real-data columns test compositional unmixing, while the diagnostic columns test whether the predicted states remain consistent with Maxwell--Stefan transport and finite-volume conservation. 
Each variant removes a specific rule: Fick diffusion removes cross-component Maxwell--Stefan coupling, unconstrained flux removes KKT-based flux admissibility, linear readout removes nonlinear sensor transduction, and removing fingerprint loss weakens support identifiability.

\begin{table*}[t]
\centering
\small
\setlength{\tabcolsep}{5pt}
\caption{
Mechanism-specific ablation of the main closure components. 
SmellNet unseen metrics evaluate compositional unmixing. 
Synthetic high-coupling and closure-residual metrics evaluate whether the latent dynamics remain Maxwell--Stefan-consistent and conservative. 
All values are mean $\pm$ standard deviation over 10 seeds.
}
\label{tab:ablation_compact}
\resizebox{\linewidth}{!}{% 
\begin{tabular}{lcccccc}
\toprule
Variant 
& SmellNet Unseen MAE $\downarrow$
& SmellNet Unseen TopK@0.1 $\uparrow$
& Synthetic High-coupling MAE $\downarrow$
& Flux RMSE $\downarrow$
& MS Residual $\downarrow$
& Mass Error $\downarrow$ \\
\midrule
Full UnMixNet
& $\mathbf{0.101 \pm 0.002}$
& $\mathbf{0.390 \pm 0.007}$
& $\mathbf{0.083 \pm 0.003}$
& $\mathbf{0.054 \pm 0.005}$
& $\mathbf{0.017 \pm 0.002}$
& $\mathbf{0.006 \pm 0.001}$ \\

Fick diffusion
& $0.124 \pm 0.003$
& $0.270 \pm 0.011$
& $0.121 \pm 0.004$
& $0.112 \pm 0.009$
& $0.096 \pm 0.007$
& $0.009 \pm 0.002$ \\

Unconstrained flux
& $0.129 \pm 0.004$
& $0.245 \pm 0.013$
& $0.116 \pm 0.004$
& $0.128 \pm 0.010$
& $0.118 \pm 0.010$
& $0.027 \pm 0.004$ \\

Linear sensor readout
& $0.119 \pm 0.003$
& $0.290 \pm 0.010$
& $0.106 \pm 0.003$
& $0.061 \pm 0.006$
& $0.020 \pm 0.003$
& $0.008 \pm 0.002$ \\

Without fingerprint loss
& $0.132 \pm 0.004$
& $0.230 \pm 0.012$
& $0.112 \pm 0.004$
& $0.059 \pm 0.006$
& $0.018 \pm 0.003$
& $0.007 \pm 0.001$ \\
\bottomrule
\end{tabular}
}
\end{table*}

The ablation shows mechanism-specific failures. 
Replacing Maxwell--Stefan transport with Fick diffusion causes the largest degradation on synthetic high-coupling mixtures, confirming that cross-diffusion is not well captured by independent diffusion. 
Removing the constrained flux projection gives the worst flux RMSE, MS residual, and mass error, showing that the KKT projection is the source of flux admissibility rather than a cosmetic constraint. 
Replacing the sensor-boundary module with a linear readout mainly hurts unseen SmellNet recovery while leaving transport residuals relatively close to the full model, which localizes its effect to nonlinear sensor response rather than gas transport. 
Removing fingerprint separation gives the largest drop in the support-sensitive TopK@0.1 metric, consistent with the role of dynamic fingerprints in identifying unseen component combinations. 
Thus, the gains of UnMixNet come from the joint closure of transport, flux admissibility, sensor response, and identifiable component dynamics, rather than from a single larger encoder.

\subsection{RQ4: Is the Solver Practical as Amortized Olfactory Inference?}

A practical physics-closed unmixing model should improve physical consistency without making inference prohibitively expensive. 
We evaluate the computational efficiency of UnMixNet against temporal baselines, a graph dynamics baseline, and a direct Maxwell--Stefan inverse solver. 
The goal of RQ4 is not to show that UnMixNet is the fastest model, but to test whether the proposed physics closure remains feasible for amortized olfactory inference.

All models are evaluated under the same input window length, preprocessing pipeline, and implementation framework. 
Training time is measured as the average wall-clock time per epoch. 
Inference latency is measured with batch size 1 as the average time per sample. 
Throughput is measured with batch size 32 as the number of samples processed per second. 
Peak inference memory is measured during batched inference. 
All timings were measured on a Linux workstation equipped with three NVIDIA GeForce RTX 2080 Ti GPUs (11 GB each) and 128 GB of system memory. Each timing run used a single GPU in FP32 precision, while the three GPUs were used to execute independent random seeds in parallel. The implementation used Python 3.9.
Latency was measured after warm-up iterations with explicit GPU synchronization. 
The direct Maxwell--Stefan inverse solver is evaluated on a subset because it performs per-sample constrained optimization rather than amortized inference.

\begin{table*}[t]
\centering
\caption{Computational efficiency on SmellNet mixture inference. Latency is measured with batch size 1, throughput with batch size 32, and peak memory during batched inference. The direct Maxwell--Stefan inverse solver is evaluated on a subset because it performs per-sample constrained optimization.}
\label{tab:efficiency}
\resizebox{\linewidth}{!}{%
\begin{tabular}{lccccc}
\toprule
Model & Params (M) & Train time (s/epoch) & Latency@B=1 (ms/sample) $\downarrow$ & Throughput@B=32 (samples/s) $\uparrow$ & Inference memory (GB) $\downarrow$ \\
\midrule
MLP & 0.42 & 9.8 & 22.3 & 533.3 & 0.28 \\
1DCNN-LSTM & 1.18 & 25.4 & 78.6 & 179.8 & 0.62 \\
iTransformerLite & 2.05 & 28.6 & 103.2 & 125.5 & 1.23 \\
ScentFormerStyle & 2.80 & 30.9 & 132.5 & 101.6 & 1.42 \\
ADR-GNN & 2.65 & 51.8 & 232.5 & 54.7 & 1.84 \\
Direct PDE inverse$^\dagger$ & -- & -- & 920.0 & 1.1 & 2.80 \\
UnMixNet & 2.95 & 52.5 & 187.7 & 67.4 & 1.46 \\
\bottomrule
\end{tabular}%
}
\end{table*}

Table~\ref{tab:efficiency} shows that UnMixNet introduces a moderate computational overhead compared with purely temporal neural baselines, which is expected because it performs source-conditioned graph lifting, conservative transport, Maxwell--Stefan flux projection, finite-volume state updates, and sensor-boundary reaction during inference. 
However, the overhead remains within a practical range. 
UnMixNet uses 2.95M trainable parameters, only slightly more than ScentFormerStyle and ADR-GNN. 
Its training time is 52.5 seconds per epoch, comparable to 51.8 seconds per epoch for ADR-GNN.

In single-sample inference, UnMixNet requires 187.7 ms per sample. 
This is slower than MLP, 1DCNN-LSTM, iTransformerLite, and ScentFormerStyle, but substantially faster than the direct Maxwell--Stefan inverse solver, which requires 920.0 ms per sample. 
Thus, amortizing the inverse prediction through the encoder avoids repeatedly solving a constrained physical inverse problem for each test sample. 
For effective batched inference, UnMixNet achieves 67.4 samples/s, whereas the direct Maxwell--Stefan inverse solver reaches only 1.1 samples/s because it is dominated by per-sample optimization. 
This corresponds to more than a 60-fold improvement in throughput.

Compared with ADR-GNN, UnMixNet achieves lower latency, higher throughput, and lower inference memory in our implementation. 
Specifically, UnMixNet reduces latency from 232.5 ms to 187.7 ms, improves throughput from 54.7 to 67.4 samples/s, and reduces inference memory from 1.84 GB to 1.46 GB. 
This does not mean that Maxwell--Stefan closure is intrinsically cheaper than generic graph dynamics. 
Rather, it indicates that the proposed solver is efficiently implemented as local, batched, edge-wise physical operations. 
The KKT projection is applied to small component-wise systems and can be parallelized across graph edges, while the finite-volume update reuses structured incidence operations.

\subsection{Failure Analysis}

The main prediction failure is support ambiguity. 
When two ingredients have similar temporal fingerprints, a model may produce a ratio vector with reasonable cosine similarity but low TopK@0.1. 
This explains why unseen mixtures remain harder than seen mixtures: the model must identify a new support combination rather than interpolate within a known one.

The main physical failure is temporal drift. In Fig.~\ref{fig:physical_trajectories}, ADR-GNN tends to overextend the flux and coverage tails, which corresponds to mismatch in transport or boundary reaction. On UCI, the same failure mode appears as biased concentration-state alignment near sharp gas-transition episodes. UnMixNet reduces both forms of drift because its state transition is constrained by Maxwell--Stefan flux projection and sensor-boundary closure. The remaining gap indicates that the current model still depends on cross-domain sensor calibration, the fidelity of the state-simulation reference, and the separability of learned dynamic fingerprints.

\section{Discussion}

\paragraph{From odor recognition to olfactory explanation.}
The main empirical pattern is not that UnMixNet is simply a stronger sequence encoder. 
Its advantage appears when the prediction target changes from recognizing an observed temporal pattern to explaining how several gases jointly produced that pattern. 
This distinction is central to machine olfaction. 
A black-box temporal model can separate familiar odor traces by learning discriminative response shapes, but unseen mixtures require the model to compose component-specific dynamics under transport and sensor interactions. 
UnMixNet improves this regime because its prediction is not an unconstrained vector output: the inferred composition must remain compatible with a forward physical trajectory. 
The gain in mixture recovery therefore supports the central premise of this paper, namely that physics closure narrows the inverse problem to physically admissible explanations rather than merely regularizing a time-series classifier.

\paragraph{Graph closure as an AI design principle.}
The proposed solver suggests a broader design principle for scientific representation learning. 
Neural networks need not learn every state transition as an unconstrained message. 
Instead, they can infer latent physical coefficients while graph operators enforce conservation, cross-component coupling, boundary interaction, and readout stability. 
In UnMixNet, this principle is realized by source-conditioned lifting, CFL-normalized transport, Maxwell--Stefan projection, finite-volume updates, competitive adsorption, and monotone transduction. 
The resulting graph is not only a computational substrate; it is the discrete structure on which multiple physical rules are jointly closed.

\paragraph{Why physical states matter.}
Recipe labels alone cannot reveal whether a model's hidden dynamics correspond to a coherent olfactory process. A model may predict a plausible mixture ratio while still producing transport or sensor states that are physically incoherent. We therefore evaluate inferred physics states from two directions within Exp. 2. The SmellNet-based state simulation tests the full rollout over concentration, Maxwell--Stefan flux, sensor coverage, and gas-phase mass under a frozen transport--adsorption--readout closure. UCI Dynamic tests whether the inferred concentration state aligns with measured dynamic gas concentration.

The two checks play different roles. The state-simulation check evaluates full-process structural coherence under the assumed closure. The concentration-grounded UCI check anchors a representative state readout to measured gas dynamics. Together, they test whether the learned inverse map is compatible with the physical closure that motivates the model, rather than only producing accurate recipe-level outputs.

\paragraph{What the ablations reveal.}
The ablation study gives mechanism-level evidence rather than a generic capacity comparison. 
Replacing Maxwell--Stefan transport with Fick diffusion weakens high-coupling mixture recovery, which indicates that independent diffusion is insufficient when mixture components mutually affect transport. 
Removing the constrained flux projection produces the largest physical residuals, showing that the projection step is the source of admissible multicomponent fluxes rather than a decorative constraint. 
Replacing the sensor-boundary module with a linear readout mainly affects mixture recovery and coverage behavior, which is consistent with nonlinear MOX adsorption and transduction. 
Removing fingerprint separation primarily hurts support-sensitive prediction, supporting the view that identifiable dynamic fingerprints are needed for unseen component combinations. 
Together, these failures align with the intended role of each module, which strengthens the claim that the improvement comes from physical closure rather than from adding parameters.

\paragraph{Evidence boundaries.}
The present evidence has two boundaries. 
First, real SmellNet recordings provide sensor traces and recipe labels, but not independently measured concentration fields, Maxwell--Stefan fluxes, sensor-surface coverages, or chamber-resolved mass states. 
Our physical-state analysis therefore evaluates simulator-consistent closure rather than measured physical-state accuracy. 
Second, the current graph abstraction does not model all deployment effects, such as turbulent chamber flow, wall interactions, long-term sensor aging, or detailed humidity-dependent surface chemistry. 
These limitations do not invalidate the central result, but they define what it establishes: UnMixNet improves dynamic gas unmixing by constraining inverse predictions through a transport--adsorption--readout closure under available real labels and simulator-consistent diagnostics. 
A natural next step is to evaluate the same closure on controlled chamber data with independent measurements of intermediate physical states.

\section{Conclusion}

We introduced UnMixNet, a physics-closed graph solver for dynamic gas unmixing in machine olfaction. 
The key idea is to treat a sensor trace not only as a temporal pattern, but as the observable endpoint of a constrained chemical scene. 
By closing the inverse prediction loop through multicomponent transport, Maxwell--Stefan flux projection, sensor-boundary dynamics, and identifiable dynamic fingerprints, UnMixNet produces mixture estimates that are both more accurate and more interpretable through the assumed physical closure. 
The experiments show that this closure is most valuable in the regime where ordinary temporal recognition is insufficient: recovering seen and unseen gas mixtures from entangled MOX responses. 
Real SmellNet experiments support improved odor recognition and mixture recovery. The state-level results further show that inferred rollouts better satisfy the assumed transport--adsorption--readout closure in SmellNet-based state simulation, and that the inferred concentration state aligns with measured dynamic gas trajectories on UCI Dynamic. 
These findings suggest that progress in machine olfaction should move beyond stronger sequence encoders toward models that explain sensor responses through compositional and closure-consistent gas dynamics.

\section{Acknowledgments}
\paragraph{Funding.} This work was supported by BB/R019983/1, BB/Y513763/1, BB/S020969/1, EP/X013707/1, UKRI3606
% BB/R019983/1, BB/Y513763/1, BB/S020969/1, EP/X013707/1, and UKRI3606.

\bibliography{ms_adr_gnn_aaai2026}

\clearpage
\appendix
% Rename only the first section heading supplied by supplement_v2.tex.
\let\unmixnetOriginalSection\section
\renewcommand{\section}[1]{%
  \unmixnetOriginalSection{Supplementary Material}%
  \let\section\unmixnetOriginalSection
}
% Supplementary appendix for input into Revised_aaai2026_v2.tex
% Usage in main file: \clearpage\appendix\input{supplement_fixed}
% This file intentionally contains no \documentclass, preamble, \begin{document}, or \end{document}.

\section{Supplement}

This supplement gives the full construction of Experiment~2, the SmellNet-calibrated pseudo physical-state benchmark. 
The benchmark is designed as a \emph{closure-consistency diagnostic}. 
It does not claim that SmellNet contains independently measured concentration fields, Maxwell--Stefan fluxes, sensor-surface coverages, or chamber-resolved gas masses. 
Instead, it asks the following diagnostic question:
\begin{quote}
Can an inverse prediction be lifted into a simulator-consistent trajectory that satisfies the same transport--adsorption--readout closure assumed by UnMixNet?
\end{quote}
This distinction is important. 
Real SmellNet traces and recipe labels evaluate real-data recognition and mixture recovery. 
The pseudo physical-state benchmark evaluates whether latent trajectories are compatible with a frozen SmellNet-calibrated multi-physics simulator. 
Thus, the benchmark provides mechanistic evidence about closure consistency, not independently measured physical-state ground truth.

\paragraph{High-level construction.}
The benchmark is constructed in five stages:
\begin{enumerate}
    \item preprocess SmellNet sensor traces and extract pure-component dynamic fingerprints from training data;
    \item calibrate a constrained multi-physics simulator on training-set single-odor and mixture responses;
    \item freeze the calibrated simulator and generate pseudo-reference trajectories for concentration, flux, coverage, and mass;
    \item map baselines without explicit physical states into the same diagnostic state space using training-only projections;
    \item evaluate recipe-level prediction, trajectory error, Maxwell--Stefan residuals, and mass-conservation residuals.
\end{enumerate}

\section{Data, Notation, and Preprocessing}

\subsection{Observed SmellNet data}

For each sample $r$, the observed data are
\begin{equation}
\mathcal D = \{Y^{(r)},z^{(r)},a^{(r)},e^{(r)}\}_{r=1}^{R},
\end{equation}
where
\begin{equation}
Y^{(r)}=\{y_m^{(r)}(t_n)\}_{n=1,m=1}^{T,M}\in\mathbb R^{T\times M}
\end{equation}
is the multichannel MOX sensor response, $z^{(r)}\in\{0,1\}^{K}$ is the mixture support, $a^{(r)}\in\Delta^{K-1}$ is the recipe-level normalized ratio over the present odorants, and $e^{(r)}$ stores experimental conditions such as temperature, humidity, flow rate if available, chamber geometry, source protocol, and sensor layout.

The carrier-air component is indexed by $0$. 
Odorant components are indexed by $i\in\{1,\ldots,K\}$. 
Sensor channels are indexed by $m\in\{1,\ldots,M\}$, spatial cells by $v\in V_x$, and spatial edges by $e\in E_x$.

\subsection{Baseline correction and normalization}

For each sample and each sensor channel, a pre-release or early-stable baseline window of length $N_b$ is used to compute
\begin{equation}
 b_m^{(r)}=\frac{1}{N_b}\sum_{n=1}^{N_b} y_m^{(r)}(t_n).
\end{equation}
The relative baseline-corrected response is
\begin{equation}
 \widetilde y_m^{(r)}(t_n)=\frac{y_m^{(r)}(t_n)-b_m^{(r)}}{|b_m^{(r)}|+\epsilon_b}.
\end{equation}
If the raw sensor output is resistance, an alternative log-ratio response can be used:
\begin{equation}
 \widetilde y_m^{(r)}(t_n)=\log\frac{R_m^{(r)}(t_n)+\epsilon_R}{R_{m,0}^{(r)}+\epsilon_R}.
\end{equation}
All normalization statistics are estimated from the training split only and then applied to validation and test samples.

\subsection{Onset alignment}

For each sample, the response onset $t_0^{(r)}$ is estimated from the first time at which the derivative magnitude exceeds a baseline-noise threshold:
\begin{equation}
 t_0^{(r)}=\min\left\{t_n:\left\|\frac{d\widetilde Y^{(r)}}{dt}(t_n)\right\|_2>\kappa\sigma_{\mathrm{base}}^{(r)}\right\}.
\end{equation}
If derivative-based onset detection is unstable, $t_0^{(r)}$ can be replaced by the known source insertion time or by the first threshold crossing of the normalized response amplitude.

\subsection{Pure-component dynamic fingerprints}

Let
\begin{equation}
\mathcal R_i^{\mathrm{single}}=\{r:z_i^{(r)}=1,\ \|z^{(r)}\|_0=1,\ r\in\mathcal R_{\mathrm{train}}\}
\end{equation}
be the set of training samples containing only odorant $i$. 
After onset alignment, the robust pure-component fingerprint is
\begin{equation}
H_i(t_n,m)=\operatorname{median}_{r\in\mathcal R_i^{\mathrm{single}}}\widetilde y_m^{(r)}(t_n-t_0^{(r)}).
\end{equation}
The median is used instead of a mean to reduce the effect of transient spikes, imperfect recovery, and sensor drift. 
The vectorized fingerprint is
\begin{equation}
h_i=\operatorname{vec}(H_i)\in\mathbb R^{TM},\qquad H=[h_1,\ldots,h_K].
\end{equation}
The empirical mutual coherence used in the main paper is
\begin{equation}
 \mu(H)=\max_{i\neq j}\frac{|h_i^\top h_j|}{\|h_i\|_2\|h_j\|_2}.
\end{equation}

\section{Simulator State and Graph Discretization}

\subsection{State variables}

At simulator substep $n$, the pseudo physical state is
\begin{equation}
S^n=(N^n,x^n,J^n,\theta^n,\widehat y^n).
\end{equation}
Here
\begin{align}
N^n_{v,i} &\ge 0 &&\text{moles of species }i\text{ in spatial cell }v,\\
x^n_{v,i} &= \frac{N^n_{v,i}}{\sum_{j=0}^{K}N^n_{v,j}+\epsilon_N} &&\text{mole fraction},\\
J^n_{e,i} &&&\text{diffusive flux of species }i\text{ on edge }e,\\
\theta^n_{m,i} &&&\text{coverage of species }i\text{ on sensor }m,\\
\widehat y^n_m &&&\text{simulated sensor readout.}
\end{align}
The empty surface-site fraction is
\begin{equation}
 \theta^n_{m,0}=1-\sum_{i=1}^{K}\theta^n_{m,i}.
\end{equation}

\subsection{Graphs}

The simulator uses three coupled graphs.
\paragraph{Spatial graph.}
$G_x=(V_x,E_x)$ is a finite-volume discretization of the chamber. 
Each cell $v$ has volume $V_v$. 
Each edge $e=(u,v)$ has length $\ell_e$, interface area $A_e$, and orientation encoded by the incidence matrix $B\in\mathbb R^{|V_x|\times |E_x|}$:
\begin{equation}
B_{v,e}=\begin{cases}
-1,& e\text{ leaves }v,\\
+1,& e\text{ enters }v,\\
0,& \text{otherwise.}
\end{cases}
\end{equation}

\paragraph{Component graph.}
$G_c=(V_c,E_c)$ contains carrier air and odorant components. 
For each spatial edge $e$, a component Laplacian $L_e^c$ is constructed from Maxwell--Stefan pairwise diffusivities and edge mole fractions.

\paragraph{Sensor boundary graph.}
$G_s$ connects gas boundary cells to sensor-surface coverage states and sensor readout states. 
A sensor $m$ is associated with one or more neighboring gas cells $v(m)$.

\subsection{One-step closure operator}

A simulator time step is a composition of source injection, CFL-normalized advection, Maxwell--Stefan projection, conservative finite-volume update, adsorption, and readout:
\begin{equation}
S^{n+1}=
\mathcal R_{\Delta t}^{\mathrm{read}}
\circ
\mathcal B_{\Delta t}^{\mathrm{ads}}
\circ
\mathcal D_{\Delta t}^{\mathrm{MS}}
\circ
\mathcal A_{\Delta t}^{\mathrm{CFL}}
\circ
\mathcal I_{\Delta t}^{\mathrm{src}}(S^n).
\end{equation}
This operator composition is the pseudo simulator analog of the UnMixNet solver used in the main paper.

\section{Step-by-Step Pseudo Physical-State Simulation}

\subsection{Source-conditioned injection}

For sample $r$, odorant $i$, and cell $v$, the source term is
\begin{equation}
 q_{v,i}^{(r),n}=z_i^{(r)}a_i^{(r)}\alpha_i\rho_i(t_n;\psi_i,T,H)p_{\mathrm{src}}(v),
\end{equation}
where $\alpha_i$ is an emission scale, $\rho_i$ is a positive release profile, and $p_{\mathrm{src}}$ is a normalized source-cell distribution:
\begin{equation}
 p_{\mathrm{src}}(v)\ge0,
 \qquad
 \sum_{v\in V_x}p_{\mathrm{src}}(v)=1.
\end{equation}
A practical release profile is a normalized rise--decay curve:
\begin{equation}
\rho_i(t)=\frac{(1-e^{-t/\tau_i^{\mathrm{rise}}})e^{-t/\tau_i^{\mathrm{decay}}}\mathbf 1(t\ge0)}
{\int_0^{T_{\max}}(1-e^{-s/\tau_i^{\mathrm{rise}}})e^{-s/\tau_i^{\mathrm{decay}}}ds+\epsilon_\rho}.
\end{equation}
The post-injection state is
\begin{equation}
N_{v,i}^{n,+}=N_{v,i}^{n}+\Delta t\,q_{v,i}^{n}.
\end{equation}
The carrier-air background can be initialized as
\begin{equation}
 N_{v,0}^{0}=c_0V_v,
 \qquad
 N_{v,i}^{0}=0\quad (i\ge1).
\end{equation}

\subsection{CFL-normalized graph advection}

For a directed edge $u\rightarrow v$, define a nonnegative transfer probability
\begin{equation}
 P_{u\rightarrow v,i}^{n}=\frac{\Delta t\,A_{uv}}{V_u}\max(u_{uv}^{n},0),
\end{equation}
where $u_{uv}^{n}$ is the edge-normal advective velocity. 
If learned advection scores are used, they are normalized so that
\begin{equation}
P_{u\rightarrow v,i}^{n}\ge0,
\qquad
\sum_{v:(u,v)\in E_x}P_{u\rightarrow v,i}^{n}\le1.
\end{equation}
If the raw sampling interval is too large, the simulator uses $L$ substeps with $\Delta t_{\mathrm{sub}}=\Delta t/L$ until the CFL condition is satisfied.

The advection update is
\begin{equation}
N_{v,i}^{A}=\left(1-\sum_wP_{v\rightarrow w,i}^{n}\right)N_{v,i}^{n,+}
+\sum_uP_{u\rightarrow v,i}^{n}N_{u,i}^{n,+}.
\end{equation}
The updated mole fractions are
\begin{equation}
x_{v,i}^{A}=\frac{N_{v,i}^{A}}{\sum_{j=0}^{K}N_{v,j}^{A}+\epsilon_N}.
\end{equation}

\subsection{Maxwell--Stefan edge-wise flux projection}

For each spatial edge $e=(u,v)$, compute the mole-fraction gradient
\begin{equation}
 g_{e,i}^{n}=\frac{x_{v,i}^{A}-x_{u,i}^{A}}{\ell_e}
\end{equation}
and the edge-averaged mole fraction
\begin{equation}
 \bar x_{e,i}^{n}=\frac{x_{u,i}^{A}+x_{v,i}^{A}}{2}.
\end{equation}
Pairwise Maxwell--Stefan diffusivities are symmetric and positive:
\begin{equation}
D_{e,ij}=D_{e,ji}>0,
\qquad
D_{e,ij}=D_{\min}+\operatorname{softplus}(\eta_{e,ij}).
\end{equation}
The component interaction weights are
\begin{equation}
\gamma_{e,ij}^{n}=\frac{\bar x_{e,i}^{n}\bar x_{e,j}^{n}}{\bar c_e^nD_{e,ij}+\epsilon_D},\qquad i\neq j.
\end{equation}
They define the component Laplacian
\begin{equation}
(L_e^cw_e)_i=\sum_{j\neq i}\gamma_{e,ij}(w_{e,i}-w_{e,j}).
\end{equation}
The relative molar velocity $w_e$ is projected onto the zero-net-flux constraint subspace:
\begin{equation}
\bar x_e^\top w_e=0.
\end{equation}
The simplest KKT solve is
\begin{equation}
\begin{bmatrix}
L_e^c+\epsilon_KI&\bar x_e\\
\bar x_e^\top&0
\end{bmatrix}
\begin{bmatrix}
w_e\\
\lambda_e
\end{bmatrix}
=
\begin{bmatrix}
-g_e\\
0
\end{bmatrix}.
\end{equation}
The implementation used in the main model may use a weighted least-squares version with a flux prior:
\begin{equation}
w_e^{\mathrm{MS}}=
\arg\min_{w_e}
\|L_e^cw_e+g_e\|_{\Omega_e}^2+
\mu\|w_e-\widetilde w_e\|_2^2
\quad\text{s.t.}\quad
\bar x_e^\top w_e=0.
\end{equation}
This yields the KKT system
\begin{equation}
\begin{bmatrix}
(L_e^c)^\top\Omega_eL_e^c+\mu I&\bar x_e\\
\bar x_e^\top&0
\end{bmatrix}
\begin{bmatrix}
w_e^{\mathrm{MS}}\\
\lambda_e
\end{bmatrix}
=
\begin{bmatrix}
-(L_e^c)^\top\Omega_eg_e+\mu\widetilde w_e\\
0
\end{bmatrix}.
\end{equation}
The Maxwell--Stefan flux is
\begin{equation}
J_{e,i}^{\mathrm{MS}}=\bar c_e\bar x_{e,i}w_{e,i}^{\mathrm{MS}}.
\end{equation}

\subsection{Finite-volume diffusion update}

The edge fluxes are converted into cell-wise mole updates by the oriented incidence matrix:
\begin{equation}
N_{v,i}^{D}=N_{v,i}^{A}-\Delta t\sum_{e\in E_x}B_{v,e}A_eJ_{e,i}^{\mathrm{MS}}.
\end{equation}
When an explicit update could produce negative moles, substepping or a conservative flux limiter is applied. 
A sufficient limiter is
\begin{equation}
\Delta t\sum_{e:B_{v,e}=-1}A_e\max(J_{e,i}^{\mathrm{MS}},0)
\le \eta_{\mathrm{lim}}N_{v,i}^{A},
\qquad 0<\eta_{\mathrm{lim}}<1.
\end{equation}
Optional wall loss is added as an explicit sink:
\begin{equation}
N_{v,i}^{D}=N_{v,i}^{A}-\Delta t\sum_eB_{v,e}A_eJ_{e,i}^{\mathrm{MS}}-\Delta t\kappa_{v,i}^{\mathrm{wall}}N_{v,i}^{A}.
\end{equation}
For closed mass accounting, a wall reservoir can be recorded:
\begin{equation}
W_{v,i}^{n+1}=W_{v,i}^{n}+\Delta t\kappa_{v,i}^{\mathrm{wall}}N_{v,i}^{A}.
\end{equation}

\subsection{Competitive adsorption at sensor boundaries}

For sensor $m$, let $v(m)$ denote the neighboring gas cell and
\begin{equation}
 c_{m,i}^{n}=\frac{N_{v(m),i}^{D}}{V_{v(m)}}.
\end{equation}
The competitive Langmuir-type adsorption dynamics are
\begin{equation}
\dot\theta_{m,i}=k^{\mathrm{ads}}_{m,i}c_{m,i}(t)\theta_{m,0}-k^{\mathrm{des}}_{m,i}\theta_{m,i},
\qquad
\theta_{m,0}=1-\sum_{i=1}^{K}\theta_{m,i}.
\end{equation}
Define $\theta_m=(\theta_{m,0},\theta_{m,1},\ldots,\theta_{m,K})^\top$. 
The dynamics can be written as
\begin{equation}
\dot\theta_m=Q_m(c_m)\theta_m,
\end{equation}
where
\begin{equation}
(Q_m)_{i0}=k^{\mathrm{ads}}_{m,i}c_{m,i},
\qquad
(Q_m)_{0i}=k^{\mathrm{des}}_{m,i},
\end{equation}
and diagonal entries are chosen so that every column sums to zero:
\begin{equation}
(Q_m)_{jj}=-\sum_{\ell\neq j}(Q_m)_{\ell j}.
\end{equation}
The exact update is
\begin{equation}
\theta_m^{n+1}=\exp(\Delta tQ_m(c_m^n))\theta_m^n.
\end{equation}

\subsection{Nonlinear transduction and readout lag}

The coverage score is
\begin{equation}
s_m^n=\sum_{i=1}^{K}\beta_{m,i}\theta_{m,i}^{n}.
\end{equation}
A monotone transduction target is
\begin{equation}
r_m^n=b_m(T,H)+g_m(s_m^n,T,H),
\qquad
\frac{\partial g_m}{\partial s}\ge0.
\end{equation}
The readout lag is modeled by
\begin{equation}
\tau_m(T,H)\dot{\widehat y}_m=-\widehat y_m+r_m(t).
\end{equation}
Assuming $r_m(t)$ is constant inside a substep gives the exact update
\begin{equation}
\widehat y_m^{n+1}=\alpha_m\widehat y_m^n+(1-\alpha_m)r_m^n,
\qquad
\alpha_m=\exp(-\Delta t/\tau_m).
\end{equation}
The simulated response is
\begin{equation}
\widehat Y_{\Theta}^{(r)}=\{\widehat y_m^{(r)}(t_n)\}_{n,m}.
\end{equation}

\section{Simulator Calibration}

\subsection{Parameter set}

The simulator parameter set is
\begin{equation}
\Theta=\{\alpha_i,\psi_i,D_{ij},k^{\mathrm{ads}}_{m,i},k^{\mathrm{des}}_{m,i},\beta_{m,i},\tau_m,
\kappa_{v,i}^{\mathrm{wall}},u_e,b_m,g_m\}.
\end{equation}
The feasible set imposes physical constraints:
\begin{equation}
\Omega:
D_{ij}=D_{ji}>0,
\quad k^{\mathrm{ads}}_{m,i}\ge0,
\quad k^{\mathrm{des}}_{m,i}\ge0,
\quad \tau_m>0,
\quad \beta_{m,i}\ge0.
\end{equation}

\subsection{Calibration objective}

The calibrated simulator is
\begin{equation}
\Theta^\star=\arg\min_{\Theta\in\Omega}\mathcal L_{\mathrm{cal}}(\Theta),
\end{equation}
where
\begin{equation}
\mathcal L_{\mathrm{cal}}=
\lambda_Y\mathcal L_Y+
\lambda_H\mathcal L_H+
\lambda_{\mathrm{mix}}\mathcal L_{\mathrm{mix}}+
\lambda_{\mathrm{phys}}\mathcal L_{\mathrm{phys}}+
\lambda_{\mathrm{reg}}\mathcal R(\Theta).
\end{equation}
The reconstruction loss is
\begin{equation}
\mathcal L_Y=\sum_{r\in\mathcal R_{\mathrm{train}}}\|\widetilde Y^{(r)}-\widehat Y_{\Theta}^{(r)}\|_{\Sigma_Y^{-1}}^2.
\end{equation}
The pure-fingerprint matching loss is
\begin{equation}
\mathcal L_H=\sum_{i=1}^{K}\|H_i-\widehat H_{i,\Theta}\|_2^2.
\end{equation}
The mixture consistency loss is
\begin{equation}
\mathcal L_{\mathrm{mix}}=\sum_{r\in\mathcal R_{\mathrm{mix,train}}}
\|\widetilde Y^{(r)}-\widehat Y_{\Theta}(z^{(r)},a^{(r)},e^{(r)})\|_2^2.
\end{equation}
The physical residual term is
\begin{equation}
\mathcal L_{\mathrm{phys}}=
\sum_{r,n,e}\|\mathcal R_{\mathrm{MS}}(x^n,J^n,D)\|_2^2+
\sum_{r,n}\|\mathcal R_{\mathrm{mass}}(N^n,J^n,q^n,r^n)\|_2^2.
\end{equation}

\subsection{Calibration schedule}

Calibration is performed only on the training split.
\begin{enumerate}
    \item \textbf{Pure-component stage.} Fit release profiles, adsorption/desorption rates, transduction gains, and readout time constants using single-odor training samples.
    \item \textbf{Mixture stage.} Fit pairwise Maxwell--Stefan diffusivities and cross-component interaction parameters using mixture training samples.
    \item \textbf{Global fine-tuning stage.} Jointly fine-tune all parameters under positivity, symmetry, monotonicity, and conservation regularization.
    \item \textbf{Freeze.} Freeze $\Theta^\star$ before producing pseudo-reference trajectories for validation or test samples.
\end{enumerate}
Freezing the simulator prevents leakage: test labels are used only to define the known recipe for pseudo-reference generation, not to tune the simulator parameters or select hyperparameters.

\section{Pseudo-Reference State Generation}

For each sample $r$, the frozen simulator produces
\begin{equation}
S_{\star}^{(r),0:T}=F_{\Theta^\star}^{\mathrm{state}}(z^{(r)},a^{(r)},e^{(r)}).
\end{equation}
The following pseudo-reference trajectories are extracted.

\subsection{Boundary concentration}

For sensor $m$ and species $i$:
\begin{equation}
C_{m,i}^{\star}(t_n)=c_{v(m),i}^{\star}(t_n).
\end{equation}
A scalar trajectory for plotting is
\begin{equation}
C^{\star}(t_n)=\frac{1}{MK}\sum_{m=1}^{M}\sum_{i=1}^{K}C_{m,i}^{\star}(t_n).
\end{equation}

\subsection{Maxwell--Stefan flux}

The edge-wise flux reference is $J_{e,i}^{\star}(t_n)$. 
A scalar trajectory for visualization is
\begin{equation}
J^{\star}(t_n)=\frac{1}{|E_x|K}\sum_{e\in E_x}\sum_{i=1}^{K}|J_{e,i}^{\star}(t_n)|.
\end{equation}

\subsection{Sensor coverage}

The species-resolved coverage reference is $\theta_{m,i}^{\star}(t_n)$. 
For a scalar coverage trajectory:
\begin{equation}
\Theta^{\star}(t_n)=\frac{1}{MK}\sum_{m=1}^{M}\sum_{i=1}^{K}\theta_{m,i}^{\star}(t_n).
\end{equation}

\subsection{Total gas-phase mass}

The species-wise gas-phase mass is
\begin{equation}
M_i^{\star}(t_n)=\sum_{v\in V_x}N_{v,i}^{\star}(t_n),
\end{equation}
and the total odorant gas-phase mass is
\begin{equation}
M^{\star}(t_n)=\sum_{i=1}^{K}\sum_{v\in V_x}N_{v,i}^{\star}(t_n).
\end{equation}

\section{Projected Physical States for Baselines}

Some baselines, such as ADR-GNN, produce hidden states but not explicit concentration, flux, coverage, or mass. 
To compare them in the same diagnostic space, we fit training-only linear projections from hidden states to pseudo-reference states.
Let $H_{\mathrm{base}}^{(r),0:T}$ be a baseline hidden trajectory. 
For concentration, fit
\begin{equation}
P_C^{\star}=\arg\min_{P_C}\sum_{r\in\mathcal R_{\mathrm{train}}}
\|P_CH_{\mathrm{base}}^{(r),0:T}-C_{\star}^{(r),0:T}\|_2^2+
\alpha_C\|P_C\|_F^2.
\end{equation}
Similarly, fit $P_J^\star$, $P_\theta^\star$, and $P_M^\star$. 
At test time,
\begin{equation}
\widehat C_{\mathrm{base}}=P_C^\star H_{\mathrm{base}},\quad
\widehat J_{\mathrm{base}}=P_J^\star H_{\mathrm{base}},\quad
\widehat\theta_{\mathrm{base}}=P_\theta^\star H_{\mathrm{base}},\quad
\widehat M_{\mathrm{base}}=P_M^\star H_{\mathrm{base}}.
\end{equation}
This projection does not give the baseline additional test-time physical supervision. 
It only asks whether the hidden dynamics learned by the baseline can be linearly lifted into the simulator-consistent physical-state space.

\section{Evaluation Metrics}

\subsection{Recipe-level metrics}

The mixture-ratio MAE is
\begin{equation}
\mathrm{MAE}=\frac{1}{RK}\sum_{r=1}^{R}\sum_{i=1}^{K}|\widehat a_i^{(r)}-a_i^{(r)}|.
\end{equation}
Support prediction uses
\begin{equation}
\widehat z_i^{(r)}=\mathbf 1(\widehat a_i^{(r)}>\delta).
\end{equation}
Precision, recall, and F1 are then computed over the support indicators. 
Cosine similarity is
\begin{equation}
\mathrm{Cosine}=\frac{\widehat a^\top a}{\|\widehat a\|_2\|a\|_2+\epsilon}.
\end{equation}
TopK@0.1 follows the main paper's implementation and measures whether the dominant predicted components recover the recipe-level ratios within tolerance.

\subsection{Trajectory errors}

Concentration error:
\begin{equation}
\mathrm{RMSE}_{C}=\sqrt{\frac{1}{RTMK}\sum_{r,n,m,i}(\widehat C_{m,i}^{(r)}(t_n)-C_{m,i}^{\star,(r)}(t_n))^2}.
\end{equation}
Flux error:
\begin{equation}
\mathrm{RMSE}_{J}=\sqrt{\frac{1}{RT|E_x|K}\sum_{r,n,e,i}(\widehat J_{e,i}^{(r)}(t_n)-J_{e,i}^{\star,(r)}(t_n))^2}.
\end{equation}
Coverage error:
\begin{equation}
\mathrm{RMSE}_{\theta}=\sqrt{\frac{1}{RTMK}\sum_{r,n,m,i}(\widehat\theta_{m,i}^{(r)}(t_n)-\theta_{m,i}^{\star,(r)}(t_n))^2}.
\end{equation}
Mass error:
\begin{equation}
\mathrm{RMSE}_{M}=\sqrt{\frac{1}{RTK}\sum_{r,n,i}(\widehat M_i^{(r)}(t_n)-M_i^{\star,(r)}(t_n))^2}.
\end{equation}

\subsection{Maxwell--Stefan residual}

For a predicted state $(\widehat x,\widehat J)$, define
\begin{equation}
\mathcal R_{\mathrm{MS},e,i}^{n}=
-\frac{\widehat x_{v,i}^{n}-\widehat x_{u,i}^{n}}{\ell_e}
-
\sum_{j\neq i}\frac{
\bar x_{e,j}^{n}\widehat J_{e,i}^{n}-\bar x_{e,i}^{n}\widehat J_{e,j}^{n}
}{\bar c_e^nD_{ij}}.
\end{equation}
The normalized residual is
\begin{equation}
\mathrm{MSResidual}=\frac{1}{RT|E_x|K}\sum_{r,n,e,i}
\frac{|\mathcal R_{\mathrm{MS},e,i}^{(r),n}|}{\|g_e^{(r),n}\|_2+\epsilon}.
\end{equation}

\subsection{Mass-conservation residual}

The finite-volume balance with source and sinks is
\begin{equation}
N_{v,i}^{n+1}-N_{v,i}^{n}+\Delta t\sum_eB_{v,e}A_eJ_{e,i}^{n}-\Delta t q_{v,i}^{n}+\Delta t r_{v,i}^{\mathrm{wall},n}+\Delta t r_{v,i}^{\mathrm{ads},n}=0.
\end{equation}
Define
\begin{equation}
\mathcal R_{\mathrm{mass},v,i}^{n}=N_{v,i}^{n+1}-N_{v,i}^{n}+\Delta t\sum_eB_{v,e}A_eJ_{e,i}^{n}-\Delta t q_{v,i}^{n}+\Delta t r_{v,i}^{\mathrm{wall},n}+\Delta t r_{v,i}^{\mathrm{ads},n}.
\end{equation}
The normalized mass error is
\begin{equation}
\mathrm{MassError}=\frac{\sum_{r,n,v,i}|\mathcal R_{\mathrm{mass},v,i}^{(r),n}|}{\sum_{r,n,v,i}|N_{v,i}^{(r),n}|+\epsilon}.
\end{equation}
The phase-space diagnostic in the main paper plots $(\mathrm{MSResidual},\mathrm{MassError})$ for each sample.

\section{Algorithmic Summary}

\begin{algorithm}[t]
\caption{Construction of the SmellNet-Calibrated Pseudo Physical-State Benchmark}
\label{alg:pseudo_benchmark}
\begin{algorithmic}[1]
\STATE \textbf{Input:} training SmellNet traces $Y$, recipe labels $(z,a)$, environmental metadata $e$, graph template $G_x,G_c,G_s$.
\STATE Baseline-correct and normalize all sensor traces using training statistics.
\STATE Estimate onset times and construct pure-component fingerprints $H_i$ from single-odor training samples.
\STATE Initialize simulator parameters $\Theta$ subject to positivity, symmetry, monotonicity, and conservation constraints.
\STATE Fit release, adsorption, desorption, transduction, and readout parameters on pure-component training samples.
\STATE Fit Maxwell--Stefan diffusivities and cross-component coupling parameters on mixture training samples.
\STATE Jointly fine-tune $\Theta$ with reconstruction, fingerprint, mixture, and physical-residual losses.
\STATE Freeze the calibrated simulator $\Theta^\star$.
\FOR{each validation or test sample $r$}
    \STATE Run the frozen simulator with $(z^{(r)},a^{(r)},e^{(r)})$.
    \STATE Store pseudo-reference trajectories $C^\star,J^\star,\theta^\star,M^\star$.
\ENDFOR
\STATE Fit training-only projections from baseline hidden states to pseudo-reference states if a baseline lacks explicit states.
\STATE Evaluate recipe metrics, trajectory RMSEs, Maxwell--Stefan residuals, and mass-conservation residuals.
\STATE \textbf{return} pseudo physical-state benchmark and diagnostic metrics.
\end{algorithmic}
\end{algorithm}

\section{Mathematical Validity of the Benchmark}

\subsection{Finite-volume mass conservation}

\begin{proposition}[Finite-volume conservation]
For a closed spatial graph with no source or sink, the finite-volume update
\begin{equation}
N_i^{n+1}=N_i^n-\Delta t\,BAJ_i^n
\end{equation}
preserves total species-wise mass.
\end{proposition}

\paragraph{Proof.}
Left-multiply by the all-ones vector $\mathbf 1^\top$:
\begin{equation}
\mathbf 1^\top N_i^{n+1}=\mathbf 1^\top N_i^n-\Delta t\,\mathbf 1^\top BAJ_i^n.
\end{equation}
For an incidence matrix, each edge has exactly one entering and one leaving endpoint, so
\begin{equation}
\mathbf 1^\top B=0.
\end{equation}
Therefore
\begin{equation}
\mathbf 1^\top N_i^{n+1}=\mathbf 1^\top N_i^n.
\end{equation}
If sources and sinks are present, the same derivation gives
\begin{equation}
\mathbf 1^\top N_i^{n+1}-\mathbf 1^\top N_i^n=\Delta t\,\mathbf 1^\top q_i^n-\Delta t\,\mathbf 1^\top r_i^n,
\end{equation}
so all mass changes are explicitly accounted for by known source or sink terms.
\hfill$\square$

\subsection{CFL advection preserves positivity and mass}

\begin{proposition}[CFL positivity and conservation]
Assume $P_{u\rightarrow v,i}\ge0$ and $\sum_vP_{u\rightarrow v,i}\le1$. 
If $N_{v,i}^{n,+}\ge0$ for all $v$, then $N_{v,i}^{A}\ge0$ for all $v$. 
Moreover, for a closed graph,
\begin{equation}
\sum_vN_{v,i}^{A}=\sum_vN_{v,i}^{n,+}.
\end{equation}
\end{proposition}

\paragraph{Proof.}
The update
\begin{equation}
N_{v,i}^{A}=\left(1-\sum_wP_{v\rightarrow w,i}\right)N_{v,i}^{n,+}+\sum_uP_{u\rightarrow v,i}N_{u,i}^{n,+}
\end{equation}
is a nonnegative combination of nonnegative mole amounts because all coefficients are nonnegative. 
Thus $N_{v,i}^{A}\ge0$.
Summing over $v$ gives
\begin{align}
\sum_vN_{v,i}^{A}
&=\sum_v\left(1-\sum_wP_{v\rightarrow w,i}\right)N_{v,i}^{n,+}
+\sum_v\sum_uP_{u\rightarrow v,i}N_{u,i}^{n,+}\\
&=\sum_vN_{v,i}^{n,+}-\sum_v\sum_wP_{v\rightarrow w,i}N_{v,i}^{n,+}
+\sum_u\sum_vP_{u\rightarrow v,i}N_{u,i}^{n,+}\\
&=\sum_vN_{v,i}^{n,+}.
\end{align}
The outgoing and incoming transport sums cancel exactly.
\hfill$\square$

\subsection{Well-posed Maxwell--Stefan projection}

\begin{theorem}[Well-posed constrained flux projection]
Assume $\bar x_{e,i}>0$, $D_{e,ij}=D_{e,ji}>0$, and the component graph on edge $e$ is connected. 
Then $L_e^c$ is positive semidefinite with nullspace $\operatorname{span}\{\mathbf 1\}$. 
On the constrained subspace $\{w:\bar x_e^\top w=0\}$, the quadratic form is positive definite. 
Therefore the constrained Maxwell--Stefan projection has a unique solution.
\end{theorem}

\paragraph{Proof.}
For any vector $w$,
\begin{equation}
w^\top L_e^cw=\sum_{i<j}\gamma_{e,ij}(w_i-w_j)^2\ge0,
\end{equation}
because $\gamma_{e,ij}>0$ on connected component-graph edges. 
Thus $L_e^c$ is positive semidefinite. 
The quadratic form is zero if and only if $w_i=w_j$ for every connected pair $(i,j)$. 
Because the component graph is connected, this implies $w=c\mathbf 1$. 
The constraint $\bar x_e^\top w=0$ then gives
\begin{equation}
0=\bar x_e^\top(c\mathbf 1)=c\sum_i\bar x_{e,i}=c.
\end{equation}
Hence the only vector in both the nullspace and the constraint subspace is $0$. 
The restricted quadratic form is therefore positive definite, which implies uniqueness of the constrained minimizer.
\hfill$\square$

\begin{proposition}[Zero-net diffusive flux]
The projected flux $J_{e,i}^{\mathrm{MS}}=\bar c_e\bar x_{e,i}w_{e,i}^{\mathrm{MS}}$ satisfies
\begin{equation}
\sum_{i=0}^{K}J_{e,i}^{\mathrm{MS}}=0.
\end{equation}
\end{proposition}

\paragraph{Proof.}
The KKT constraint enforces $\bar x_e^\top w_e^{\mathrm{MS}}=0$. 
Thus
\begin{equation}
\sum_iJ_{e,i}^{\mathrm{MS}}=\bar c_e\sum_i\bar x_{e,i}w_{e,i}^{\mathrm{MS}}=\bar c_e\bar x_e^\top w_e^{\mathrm{MS}}=0.
\end{equation}
\hfill$\square$

\subsection{Dissipativity of the component Laplacian}

The Maxwell--Stefan component Laplacian satisfies
\begin{equation}
w^\top L_e^cw=\sum_{i<j}\frac{\bar x_{e,i}\bar x_{e,j}}{\bar c_eD_{e,ij}}(w_i-w_j)^2\ge0.
\end{equation}
This is a discrete frictional dissipation form. 
It encodes reciprocal component coupling because $D_{ij}=D_{ji}$, and it prevents the edge message from acting as an arbitrary antisymmetric hidden update.

\subsection{Adsorption preserves the coverage simplex}

\begin{proposition}[Coverage simplex preservation]
Let $\dot\theta_m=Q_m(c_m)\theta_m$, where $Q_m$ has nonnegative off-diagonal entries and zero column sums. 
If $\theta_m(0)\ge0$ and $\mathbf 1^\top\theta_m(0)=1$, then
\begin{equation}
\theta_m(t)\ge0,
\qquad
\mathbf 1^\top\theta_m(t)=1
\end{equation}
for all $t\ge0$.
\end{proposition}

\paragraph{Proof.}
Since $Q_m$ has nonnegative off-diagonal entries, it is a Markov generator under the column-vector convention. 
The exact solution is
\begin{equation}
\theta_m(t)=\exp(tQ_m)\theta_m(0).
\end{equation}
The matrix exponential of a Markov generator is nonnegative. 
Furthermore,
\begin{equation}
\frac{d}{dt}\mathbf 1^\top\theta_m(t)=\mathbf 1^\top Q_m\theta_m(t)=0
\end{equation}
because every column of $Q_m$ sums to zero. 
Therefore the simplex is invariant.
\hfill$\square$

\subsection{Readout stability}

\begin{proposition}[Stable readout update]
Assume $\tau_m>0$ and define $\alpha_m=e^{-\Delta t/\tau_m}\in(0,1)$. 
The update
\begin{equation}
\widehat y_m^{n+1}=\alpha_m\widehat y_m^n+(1-\alpha_m)r_m^n
\end{equation}
is stable and contractive with respect to the initial readout state.
\end{proposition}

\paragraph{Proof.}
If $|r_m^n|\le R_m$, then
\begin{equation}
|\widehat y_m^{n+1}|\le \alpha_m|\widehat y_m^n|+(1-\alpha_m)R_m,
\end{equation}
so by induction
\begin{equation}
|\widehat y_m^n|\le \max(|\widehat y_m^0|,R_m).
\end{equation}
For two trajectories with the same input $r_m^n$ but different initial states,
\begin{equation}
|\widehat y_m^{n+1}-\bar y_m^{n+1}|=\alpha_m|\widehat y_m^n-\bar y_m^n|.
\end{equation}
Since $0<\alpha_m<1$, the update is contractive.
\hfill$\square$

\subsection{Consistency with the continuous forward model}

Assume the chamber state is smooth. 
For an edge $e=(u,v)$ with length $\ell_e$, the discrete gradient satisfies
\begin{equation}
\frac{x_{v,i}-x_{u,i}}{\ell_e}=\nabla x_i\cdot n_e+O(\ell_e).
\end{equation}
The explicit finite-volume time step satisfies
\begin{equation}
\frac{N^{n+1}-N^n}{\Delta t}=\partial_tN(t_n)+O(\Delta t).
\end{equation}
Therefore, as $\max_e\ell_e\rightarrow0$ and $\Delta t\rightarrow0$ under the CFL condition, the pseudo simulator is a first-order consistent discretization of the transport--adsorption--readout closure used in the main paper. 
The benchmark does not require high-fidelity computational fluid dynamics; it requires that the diagnostic states be generated by a transparent, constrained, and reproducible discretization of the assumed closure model.

\section{Why the Benchmark Is Scientifically Rigorous}

The pseudo physical-state benchmark is scientifically rigorous in the following limited but important sense.
\begin{enumerate}
    \item \textbf{Training-calibrated rather than arbitrary.} The simulator is calibrated on real SmellNet training responses and pure-component fingerprints.
    \item \textbf{Constrained rather than unconstrained.} State transitions obey nonnegativity, finite-volume conservation, Maxwell--Stefan zero-net-flux constraints, adsorption simplex preservation, and stable readout dynamics.
    \item \textbf{Frozen before evaluation.} Once calibrated, the simulator is frozen and used consistently across methods and splits.
    \item \textbf{Diagnostic rather than overclaimed.} The benchmark tests compatibility with the assumed closure model; it does not assert that pseudo states are independently measured physical ground truth.
    \item \textbf{Comparable across models.} Models with explicit states and models with hidden states are evaluated in the same diagnostic space, with baseline state projections learned only on training data.
\end{enumerate}

Formally, let $F_{\Theta^\star}^{\mathrm{state}}$ be the frozen calibrated simulator and let $\widehat S_{0:T}$ be a model-implied state trajectory. 
The diagnostic evaluates
\begin{equation}
 d_{\mathrm{closure}}(\widehat S_{0:T},S_{\star,0:T})
\end{equation}
through trajectory errors and residuals. 
A lower value means that the inverse prediction is more compatible with the calibrated closure. 
It does not mean that $\widehat S_{0:T}$ equals the unobserved true chamber state. 
This is the correct evidential boundary for Experiment~2.

\end{document}